\documentclass[11pt]{article} 

\usepackage[total={6.5in, 8in}]{geometry}

\usepackage[backend=bibtex]{biblatex}
\usepackage{algorithm}
\usepackage{algorithmic}
\usepackage{comment}

\usepackage{mathtools}
\DeclarePairedDelimiter{\ceil}{\lceil}{\rceil}

\usepackage{booktabs}
\usepackage{subcaption}
\usepackage{amsthm}
\usepackage{amsmath, amsfonts}
\newtheorem{theorem}{Theorem}
\newtheorem{lemma}{Lemma}

\theoremstyle{definition}
\newtheorem{definition}{Definition}[section]

\usepackage{hyperref}
\hypersetup{
   breaklinks=true,   
   hidelinks,
   pdfusetitle=true,  
}

\usepackage{newfloat}
\usepackage{listings}

\title{Liquid Democracy for Low-Cost Ensemble Pruning}

\author{
  Ben Armstrong\footnote{Corresponding Author}\\
  University of Waterloo\\
  \texttt{ben.armstrong@uwaterloo.ca}
  \and
  Kate Larson\\
  University of Waterloo\\
  \texttt{kate.larson@uwaterloo.ca}
}
\date{}

\newcommand{\BibTeX}{\rm B\kern-.05em{\sc i\kern-.025em b}\kern-.08em\TeX}

\begin{document}

\maketitle 

\begin{abstract}
We argue that there is a strong connection between ensemble learning and a delegative voting paradigm  - \textit{liquid democracy} -- that can be leveraged to reduce ensemble training costs.
We present an incremental training procedure that identifies and removes redundant classifiers from an ensemble via delegation mechanisms inspired by liquid democracy. Through both analysis and extensive experiments we show that this process greatly reduces the computational cost of training compared to training a full ensemble. By carefully selecting the underlying delegation mechanism, weight centralization in the classifier population is avoided, leading to higher accuracy than some boosting methods.
Furthermore, this work serves as an exemplar of how frameworks from computational social choice literature can be applied to problems in nontraditional domains. 
\end{abstract}


\section{Introduction}

In the past several years, the training of machine learning systems has consumed increasingly large amounts of data and compute. In the search for ever-improving performance, models have grown larger, more data has been collected, and the cost of machine learning has grown while performance only improves incrementally \cite{patterson2021carbon}. This leads to negative repercussions affecting privacy by incentivizing mass data collection, increased development time due to the time taken to train models, and significant environmental costs. It also limits access to the best-performing models to those groups with enough resources to support storing massive amounts of data and training large models. 
Recent advances have begun to consider learning from few examples for settings where data is hard to generate or resources are limited \cite{wang2020generalizing} however this field is still in its early stages. We propose adapting an existing paradigm of opinion aggregation to address the problem of compute requirements during classifier ensemble training.

Ensemble learning for classification has long studied the problem of combining class predictions from groups of classifiers into a single output prediction. Condorcet's Jury Theorem, a well-known result from social choice theory (predating ML research by 2 centuries), states that if voters attempt to guess the correct outcome of some ground-truth decision then the majority vote is increasingly likely to be correct as voters are added if all voters are independent and have accuracy above 0.5 \cite{nicolas1785essai}. This situation corresponds very closely with ensemble learning - classifiers try to predict the correct class of some example, often choosing from between 2 classes. The most significant difference is that, in ensemble learning, classifiers are typically correlated. This has lead to a great deal of research exploring the effect of diversity between ensemble members on ensemble performance \cite{kuncheva2003measures}.

This parallel demonstrates the strong connection between ensemble learning and social choice. Other social choice results have also been applied to ensemble learning \cite{cornelio2021voting} with moderate success; however, while a variety of standard voting rules can occasionally improve ensemble performance over majority vote or other ensembling methods they do not address compute requirements and may require additional compute.
Further, while in practice ensembles do tend to improve with more members this effect is limited and Condorcet's Jury Theorem does not directly apply as each new classifier is typically trained on data drawn from the same source and is not entirely independent from other classifiers in the ensemble.

This paper introduces a connection between ensemble learning and an emerging social choice paradigm -- \textit{liquid democracy} -- in an attempt to mitigate the issues outlined above: increasing training costs, and the possibility of having multiple highly similar classifiers in an ensemble. Liquid democracy is a process of transitive, delegative voting wherein all voters may choose whether to delegate their vote (and all their received delegations) to another voter or to vote directly with weight corresponding to the number of delegations they have received. We use liquid democracy to allow classifiers to ``delegate'' their classification decisions to higher-performing classifiers within their ensemble. This paper describes a process of incremental training and delegation that dramatically reduces the training and inference cost of large ensembles without impacting their performance. Our approach provides several parameters which allow optimizing for either training cost or accuracy.

The remainder of the paper is structured as follows: This section concludes with an overview of the related work. Section 2 introduces our model formally, discusses the parallel between ensemble learning and voting in social choice, and explains the metrics we are evaluating. Section 3 analyzes our model and describes the expected reduction in cost from our training method. Our experimental results are presented in Section 4 and Section 5 concludes with final remarks and a discussion of future directions for our method.

\subsection{Related Work}

This work draws upon two distinct lines of research - liquid democracy and ensemble methods. We review the most relevant material from each field.

\subsubsection{Liquid Democracy} In liquid democracy, a number of papers have explored a setting wherein voters aim to find ground truth on a single issue while we work in a setting wherein voters make ground truth decisions across many issues. A theoretical line of research has established that finding the delegations which maximize group accuracy is NP-hard even to approximate \cite{kahng2018liquid,caragiannis2019contribution,becker2021can,halpern2021defense}. 
However, an empirical line of research demonstrates that in practice even some simple delegation mechanisms typically outperform direct voting across a wide range of settings \cite{alouf2022how,becker2021can}. We are aware of one prior paper applying liquid democracy to ensemble learning which found little benefit to accuracy from only a single round of delegation \cite{armstrong2021limited}. In this work we explore multiple rounds of delegation and focus on metrics beyond accuracy such as training cost and F1-score. 
Another area of work related to ours studies liquid democracy from the perspective of binary aggregation \cite{christoff2017binary}. In many ways this model is similar to our own, however, voters are able to delegate on individual issues (whereas classifiers in our setting delegate entirely to a single other classifier) and the focus of analysis is on rationality of delegations.

\subsubsection{Ensemble Pruning Methods}

The potential benefit to accuracy from pruning ensembles was first proven by \citeauthor{zhou2002ensembling}~\cite{zhou2002ensembling}. Since then, many pruning strategies have been explored. One survey divides most approaches into \textit{ranking} or \textit{search-based} methods \cite{sagi2018ensemble}. Ranking methods -- which this work uses -- score each ensemble member based on some metric, often some combination of diversity and accuracy, then remove the lowest scoring members. A key distinction of our approach is the method of transferring weights from a pruned classifier to a remaining classifier. A different method of weight transfer has been explored which shifts the weight of a pruned classifier to remaining classifiers based on their similarity. However, the approach shows little positive effect from the weight shifts \cite{tamon2000boosting}.
Pruning in an incremental learning setting is explored in one paper which develops a method of creating a new ensemble for each increment of data which is subsequently pruned to reach a single final ensemble \cite{zhao2011incremental}.

\section{Delegative Ensemble Pruning}

Our model bridges two fields of research - machine learning and social choice. We first describe our algorithm and the learning problem we focus on, then connect it to the social choice framework guiding our research.

We introduce a new algorithm for incrementally pruning and reweighting an ensemble during the training process in order to reduce computational costs of training while maintaining or improving accuracy.
At the high level, during each time step our algorithm removes some number of classifiers and increases the prediction weight of some remaining classifiers. Then the remaining classifiers are trained on a new subset of data. This process stops only when the ensemble reaches some pre-determined size or runs out of training data.
In this way, fewer classifiers are trained at each increment and less compute is spent on training than if the full ensemble were to be trained.

In more detail, we train an ensemble of classifiers $V =$ $ \{v_1, v_2,$ $..., v_n\}$ on a dataset of size $m$ partitioned into $T$ subsets of roughly equal size, $(X^\text{train}, Y^\text{train}) = \{X^\text{train}_1, Y^\text{train}_1, ..., X^\text{train}_T, Y^\text{train}_T\}$. Weight vector $W_t$ contains the weight associated with each classifier's predictions at increment $t$, initially set to $1$; $W_0 = \lbrack 1 \rbrack^n$. 
During training, an estimate of the accuracy of each non-pruned classifier is continually refined. A = $\lbrack a_{t, i}\rbrack$ stores the training accuracy of each classifier $v_i$ at each increment $t$ while Q = $\lbrack q_{t, i}\rbrack$ contains the training accuracy of each classifier averaged over all previous increments; $q_{t, i} = \frac{1}{t}\sum_{s \leq t} a_{s, i}$.
Typically we refer to the current weight and estimated accuracy of classifier $v_i$ as $w_i$ and $q_i$ respectively, only including a time subscript where necessary. The class receiving predictions with the highest summed weight is the output of the ensemble during inference.

During each time step $t$, each classifier with a non-zero weight is trained on $(X^\text{train}_t, Y^\text{train}_t)$.
After training, a fixed proportion $r$ of classifiers remain in the ensemble while $1-r$  are removed (by setting their weight to $0$). For each classifier $v_i$ removed, some other classifier has its weight increased by $w_i$. Thus, total weight in the system remains constant at all times.
We refer to the number of remaining classifiers at the end of each time step as $n^\text{final}_t$ and parameterize the minimum number of classifiers in the ensemble as $n^\text{final}$.

The above loop of training the ensemble then pruning and reweighting is continued until one of two conditions is met: (1) All $T$ subsets of data have been used, or (2) $n^\text{final}$ classifiers remain with non-zero weight. Finally, once finished, each classifier remaining in the ensemble is fully trained on the entire training set. This final step is optional; while it improves the accuracy of the ensemble it also increases the total training time.

\autoref{alg:delegation_algorithm} summarizes the above procedure. An example of this delegative pruning process on a small ensemble is shown in \autoref{fig:ensemble_delegation_diagram}. At each increment a single classifier is ``removed'' by setting its weight to $0$ and the estimated accuracy of each remaining classifier is updated based on its performance on $(X^\text{train}_t, Y^\text{train}_t)$. The final ensemble contains only $60\%$ of the classifiers in the original ensemble, leading to fewer classifiers that must be fully trained.

\begin{figure*}[t]
    \centering
    \includegraphics[width=0.8\textwidth,keepaspectratio]{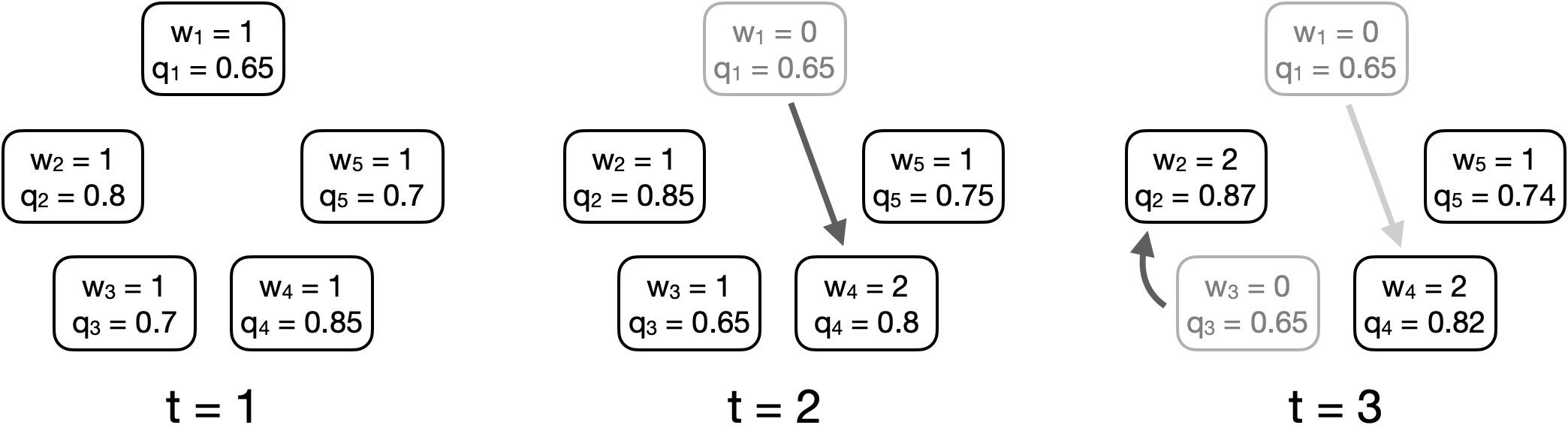}
    \caption{A possible outcome of incremental pruning on a small ensemble over 3 time steps. At each increment, each remaining classifier updates its average accuracy according to its performance on the most recent training data. At $t = 1$, each classifier is weighted equally and trained on one increment of data to generate an estimate of their accuracy. $v_1$ has the lowest estimated accuracy while $v_4$ has the highest. At $t = 2$, the weight of $v_1$, the ``weakest'' classifier during $t = 1$ is transferred to the ``strongest'' classifier, $v_4$ and the 4 classifiers remaining with non-zero weight train on another increment of data and estimates of their accuracy are refined. Finally, at $t = 3$, as $v_3$ is the classifier with the lowest accuracy estimate its weight is transferred to the more accurate $v_1$. After two increments of pruning, future training of the ensemble will use only 60\% of the original training cost. Note: The actual choice of which classifiers are removed and where their weight is transferred depends upon parameters described in \autoref{sec:delegation_mechanisms}.}
    \label{fig:ensemble_delegation_diagram}
\end{figure*}

\subsection{Ensemble Pruning as Liquid Democracy}
\label{sec:pruning_as_ld}

We draw a parallel between ensemble classification and voting. Ensembles can be seen as groups of voters deciding on the outcome of an election with some objectively correct outcome \cite{cornelio2021voting,campagner2020ensemble}. Each classification task can be seen as an election where $n$ classifiers ``vote'' on the correct class of an example. So ensemble learning can be seen as $m$ elections which all use the same set of ``voters'' (i.e., classifiers)

Thus, we can express our model in social choice terms as a setting of $n$ voters $V = \{v_1, ..., v_n\}$. The probability that $v_i$ is correct on any given election, or their accuracy, is denoted $q_{i}$ and corresponds to a classifier's training accuracy. We denote the set of voters with higher accuracy than $v_i$ as $N^+(v_i) = \{v_j \in V | q_{j} > q_{i} \}$.

Our pruning procedure is based upon delegative voting in liquid democracy (LD) \cite{paulin2020overview}. Within LD, voters are either \textit{delegators} or \textit{representatives}; representatives are voters that vote directly and are modelled as delegating to themselves. Delegators delegate their vote to another voter who then represents them. If delegators have themselves received delegations, they also delegate all of the delegations they have received. Each representative votes with a weight equal to the number of delegations they have received. When viewing classifiers as voters, a delegation serves both to prune a single classifier from the ensemble and to increase the weight of one of the classifiers remaining in the ensemble.

More formally, a delegation function $d: V \rightarrow V$ gives the delegation of each voter. $d(v_i) = v_j$ indicates that $v_i$ delegates to $v_j$. We say that $v_i$ votes directly by a self-delegation, $d(v_i) = v_i$. Delegation applies transitively so that a delegation might ``travel'' several hops before reaching a voter that votes directly. $d^*(v_i)$ refers to the repeated application of $d(v_i)$ until a fixed point (a self-delegation) is reached. The representative of $v_i$ is $d^*(v_i)$.
Throughout this paper we disallow any delegation that would result in a cycle.

When equating classifiers and voters, delegation can be seen as a means of ensemble pruning. At any time during training the representatives, denoted $G(V) = \{d^*(v_i) | v_i \in V\}$, are exactly the set of classifiers in an ensemble that have not been pruned. A classifier is pruned by making it delegate to another classifier, which simultaneously removes it from the ensemble and increases the weight of its new representative. The weight of each classifier is determined by whether they delegate and how many delegations they receive. That is, the weight of each $v_i \in V$ is

\[
w_{i} = 
\begin{cases}
0 & \text{if } v_i\not \in G(V) \\
|\{v_j \in V | v_i = d^*(v_j)\}| & \text{otherwise }
\end{cases}
\]

By using various delegation functions and treating classifiers in an ensemble as though they are voters we can explore several approaches to pruning ensembles that have a basis in social choice. The particular delegation functions we use are described in \autoref{sec:delegation_mechanisms}. Hereafter we refer to classifiers and voters interchangeably.

\begin{algorithm}[tb]
\caption{Training and Pruning Algorithm. \texttt{fit} and \texttt{partial\_fit} methods refer to methods in the sklearn library for each model type \cite{scikit-learn}. \texttt{select\_pruned\_clfs} and \texttt{transfer\_weight} are defined in \autoref{sec:delegation_mechanisms}.}
\label{alg:delegation_algorithm}
\textbf{Input}: $V, r, n^\text{final}, T, X^\text{train}, Y^\text{train}$
\begin{algorithmic}[1] 
\STATE $W \gets \lbrack 1 \rbrack^n$
\STATE $Q \gets \lbrack 1 \rbrack^n$
\STATE $X^\text{train}, Y^\text{train} \gets \{X^\text{train}_1, Y^\text{train}_1, ..., X^\text{train}_T, Y^\text{train}_T\}$
\FOR{$t \in T$}

    \FOR{$v_i \in V$}
        \IF {$w_i \neq 0$}
        \STATE $\texttt{partial\_fit(}v_i, \text{ } X^\text{train}_t, \text{ }Y^\text{train}_t\texttt{)}$
        \STATE $Q_{t, i} \gets \texttt{mean(}\{q_{s, i} \text{ } \forall s \leq t\}\texttt{)}$
        \ENDIF
    \ENDFOR

\STATE $V_t^\text{remove} \gets \texttt{select\_pruned\_clfs(}V, Q, W, r\texttt{)}$

\IF {$V_t^\text{remove} = \emptyset$}
    \STATE \texttt{break}
\ENDIF

\FOR{$v_i \in V_t^\text{remove}$}
    
    \IF {$|\{w | w \neq 0 \forall w \in W\}| \leq n^\text{final}$}
        \STATE \texttt{break}
    \ENDIF
    \STATE $\texttt{transfer\_weight(} w_i, V, W, Q\texttt{)}$
    \STATE $w_i \gets 0$
\ENDFOR

\ENDFOR

\FOR{$v_i \in V$}
        \IF {$w_i \neq 0$}
        \STATE $\texttt{fit(}v_i, \text{ } X^\text{train}, \text{ }Y^\text{train}\texttt{)}$
        \ENDIF
    \ENDFOR
\end{algorithmic}
\end{algorithm}

\section{Delegation Mechanisms}
\label{sec:delegation_mechanisms}

Here we introduce delegation mechanisms. We include two baseline mechanisms -- Direct (no delegation) and Random -- as well as mechanisms such as Random Better and Max that have been explored throughout previous work \cite{kahng2018liquid,alouf2022how} and new mechanisms: Proportional Better and Proportional Weighted.

\begin{definition}[Delegation Mechanism]

A delegation mechanisms is a tuple $\mathcal{D} = (g, p)$ consisting of two functions:

\begin{enumerate}
    \item A \textbf{delegator selection} function $g: 2^V \rightarrow 2^V$ selects $n^\text{final}_t$ representatives who will delegate in increment $t$. This function implements the \texttt{select\_pruned\_clfs()} method on line 7 in \autoref{alg:delegation_algorithm}.
    
    \item A \textbf{delegation probability} function $p: G(V) \times G(V) \rightarrow \mathbb{R}_{\geq 0}$ which accepts as input two representatives $v_i$ and $v_j$ and determines the probability that $v_i$ will delegate to $v_j$. This corresponds to the \texttt{transfer\_weight()} method on line 15 of \autoref{alg:delegation_algorithm}.
\end{enumerate}
\end{definition}

\subsection{Delegator Selection Functions}

We explore three delegator selection functions; two baselines $g^\text{random}$ and $g^\text{direct}$ as well as a more intelligent function $g^\text{worst}$ each defined below. The Random delegation mechanism uses $g^\text{random}$ and the Direct mechanism uses $g^\text{direct}$; \textit{all other delegation mechanism} use $g^\text{worst}$.

\begin{definition}
    Delegator Selection Function $g^\text{random}(V, d)$ selects $n^\text{final}_t$ representatives uniformly at random.
\end{definition}

\begin{definition}
    Delegator Selection Function $g^\text{direct}(V, d)$ selects no representatives and performs no delegation; return $\emptyset$.
\end{definition}

\begin{definition}
    Delegator Selection Function $g^\text{worst}(V, d)$ selects the $n^\text{final}_t$ representatives with the lowest $q$ (lowest average training accuracy over all previous increments).
\end{definition}

\subsection{Delegation Probability Functions}
\label{sec:delegation_mechanism_descriptions}

Delegation probability functions determine the probability that one representative will delegate to a given voter (not necessarily a representative). Each delegation mechanism we explore uses a different delegation probability function, corresponding to the function's name (e.g. the Max delegation mechanism makes use of $p^\text{max}$).

At each increment, a delegator selection function selects a number of representatives which will be turned into delegators. For each of those representatives, the delegation probability function calculates the probability of that representative delegating to each other voter (excluding themselves). Below we describe intuitively and mathematically how probabilities are calculated for each mechanism.

\begin{definition}

The \textbf{Direct} Delegation Probability Function selects no delegates and has all voters represent themselves. This is equivalent to the situation where each voter delegates to themselves. Note that this function is redundant as $g^\text{direct}$ does not select any new delegators.

\[
p^\text{direct}(v_i, v_j) = \begin{cases} 
      1 & v_i = v_j \\
      0 & otherwise
   \end{cases}
\]
\end{definition}

\begin{definition}

The \textbf{Random} Delegation Probability Function selects delegates uniformly at random for new delegators.

\[
p^\text{random}(v_i, v_j) = \frac{1}{n}
\]
\end{definition}

\begin{definition}

The \textbf{Max} Delegation Probability Function distributes delegated weight as evenly as possible among voters.
Let $h(v_i) = \{q_{j} | w_{j} \leq w_{k} \forall v_j, v_k \in N^+(v_i)\}$ denote the competencies of voters in $N^+(v_i)$ with minimal weight. The least accurate voters delegate their weight to the most accurate voter in $h(v_i)$, with ties broken randomly. This differs from existing ``max'' delegation mechanisms \cite{bloembergen2019rational} by selecting the least accurate voters to delegate then spreading their weight evenly among potential representatives; in our setting, delegating only to the most accurate voter would quickly cause a single voter to have the majority of the weight.

\[
p^\text{max}(v_i, v_j) = \begin{cases} 
      1 & q_j = \text{max q} \in h(v_i) \\
      0 & \text{otherwise}
   \end{cases}
\]
\end{definition}

\begin{definition}

The \textbf{Random Better} Delegation Probability Function selects a delegate for each delegator uniformly at random from voters with strictly higher accuracy \cite{kahng2018liquid}.

\[
p^\text{rand\_better}(v_i, v_j) = \begin{cases} 
      \frac{1}{|N^+(v_i)|} & j \in N^+(v_i) \\
      0 & \text{otherwise}
   \end{cases}
\]

\end{definition}

\begin{definition}

The \textbf{Proportional Better} Delegation Probability Function has delegators delegate to a representative with higher accuracy, however, the chance of delegating to a voter is directly correlated with the difference between their accuracy and that of the delegator. Shown here are delegation probabilities before normalization.

\[
p^\text{prop\_better}(v_i, v_j) \propto \begin{cases} 
      \frac{q_j - q_i}{\sum_{v_k \in N^+(v_i)} q_k-q_i } & j \in N^+(v_i) \\
      0 & \text{otherwise}
   \end{cases}
\]
\end{definition}

\begin{definition}

The \textbf{Proportional Weighted} Delegation Probability Function returns delegation probabilities based on both the accuracy difference between delegator and delegatee, as well as the weight of the representative ultimately being delegated to. A lower weight leads to a higher delegation probability. Shown here are delegation probabilities before normalization.

\[
p^\text{prop\_weighted}(v_i, v_j) \propto \begin{cases} 
      \frac{1}{w_{d^*(j)}} \frac{q_j - q_i}{\sum_{v_k \in N^+(v_i)} q_k-q_i } & j \in N^+(v_i) \\
      0 & \text{otherwise}
   \end{cases}
\]

\end{definition}

\section{Chance of Harm from Delegation}

While delegation can improve performance, it also brings the possibility of harm by reducing ensemble accuracy. By reducing the number of active voters (that is, the number of classifiers in the ensemble), the ensemble may approach dictatorships where a single voter has a controlling weight and is less accurate than if input from multiple voters was required to agree in order to make decisions.
To give intuition about the extent to which delegations are ``safe'' and unlikely to reduce accuracy of an ensemble we consider the number of initial delegation states (those where $d^*(v_i) = v_i$ $\forall v_i \in V$) where any \textit{single} delegation will lead to a reduction in group accuracy. We show that such ``harmful'' states exist but that they are vanishingly small in frequency and the vast majority of possible states can weakly benefit from delegation.

Consider the votes of a set of voters with no delegations. Since each one has equal weight, these can be treated as an $n \times m$ binary matrix $P$ where $p_{ij} = 1$ if voter $i$ voted correctly on the $j^{\text{th}}$ example and 0 otherwise, respectively shown with a green check mark and a red x in \autoref{fig:pivotal_voter_diagrams}. Any column with $\ceil{\frac{n}{2}}$ or more $1$'s (check marks) indicates the corresponding example is classified correctly. 
An example $j$ is called \textit{pivotal} if $\sum_{i = 1}^{n} p_{ij} = \ceil{\frac{n}{2}}$. That is, if it is correctly classified by a minimum margin. Any voter that is correct on a pivotal example is said to be a \textit{pivotal voter} on that example, and may be pivotal on several examples. Similarly, if a voter is incorrect on example $j$ where $\sum_{i = 1}^{n} p_{ij} = \ceil{\frac{n}{2}}-1$ they are referred to as an \textit{incorrect pivotal voter}.

We can now establish an upper bound on the number of states in which any individual delegation would result in a decrease in accuracy (and thus, well-informed and well-meaning voters would make no delegations).

\begin{figure}[t]
     \centering
     \begin{subfigure}[b]{0.45\textwidth}
         \centering
         \includegraphics[width=\textwidth]{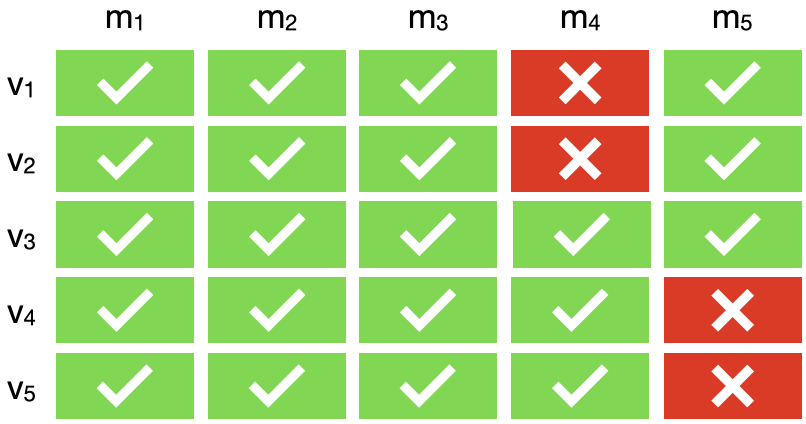}
         \caption{}
         \label{fig:2_pivotal_voters}
     \end{subfigure}
     \hfill
     \begin{subfigure}[b]{0.45\textwidth}
         \centering
         \includegraphics[width=\textwidth]{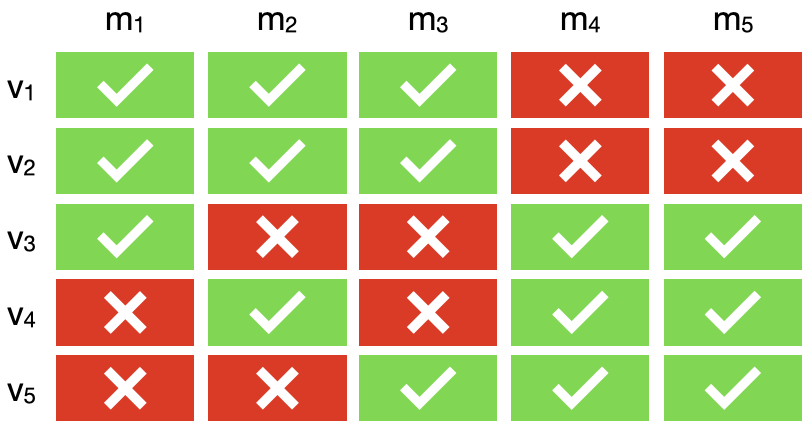}
         \caption{}
         \label{fig:m_pivotal_voters}
     \end{subfigure}
     
     \caption{Two possible states when an ensemble composed of 5 classifiers predicts the classes of 5 examples. A cell shows whether a particular voter (rows) is correct or incorrect in classifying each example (columns). \textbf{(a)} All voters are pivotal and only examples $m_4$ and $m_5$ are pivotal. If any non-pivotal examples $m_1$, $m_2$, or $m_3$ were removed all voters in \textbf{(a)} would remain pivotal. \textbf{(b)} All voters are pivotal and all examples are pivotal.}
     \label{fig:pivotal_voter_diagrams}
\end{figure}

\begin{lemma}\label{lemma:pivotal_voters}
    If every single delegation reduces group accuracy then each voter must be pivotal on at least one example.
\end{lemma}
\begin{proof}
This can be simply shown by contradiction: By definition, if there exists a single voter who is not pivotal on any example, they can delegate to any other voter without causing any examples to change from being classified correctly to being classified incorrectly.
\end{proof}

Note that this lemma does not hold in reverse. If every voter is pivotal on at least one example, it is not the case that delegation must reduce group accuracy. For example, observe that $v_1$ may delegate to $v_2$ in both parts of \autoref{fig:pivotal_voter_diagrams} without changing group accuracy.

When considering the harmful effects of a single delegation we need only focus on pivotal examples as they are the only correctly classified examples that may become incorrectly classified. Thus the classification decisions on non-pivotal examples are irrelevant to this problem and our theorem considers only the states that may occurs within pivotal examples.

\begin{theorem}\label{thm:bad_delegation_counts}

    In an ensemble with $n$ voters and $m$ examples, the total number of ways in which classification decisions can be made on pivotal examples such that every voter is pivotal is $s_{n, m}^\text{pivotal} = \sum_{m_p=2}^{m} \binom{n}{\ceil{\frac{n}{2}}}^{m_p}$. Without any restrictions the same examples could have 
    $s_{n, m}^\text{total} = \sum_{m_p=2}^{m} 2^{nm_p}$ possible states.
\end{theorem}

The calculations required to support this theorem are found in \autoref{appendix:weak_improvement}.
By counting the number of states in which every voter may be pivotal and comparing with the total number of ways the same examples may be classified we can determine of the fraction of total states in which every voter is pivotal. From \autoref{lemma:pivotal_voters}, this will overcount the fraction of states where any single delegation reduces group accuracy.
Note that when calculating $s_{n, m}^\text{pivotal}$ and $s_{n, m}^\text{total}$ we are counting only the (pivotal or total) number of states on $m_p$ columns/examples, and excluding the number of ways in which the $m - m_p$ non-pivotal examples may be classified (which is irrelevant to counting pivotal states). This enables us to compare $s_{n, m}^\text{pivotal}$ and $s_{n, m}^\text{total}$ without evaluating the total number of possible states in either case.

\autoref{tab:harmful_state_upper_bound} shows this loose upper bound on the proportion of states in which a single delegation is necessarily harmful as $n$ and $m$ grow. With even moderately-sized datasets and ensembles it becomes nearly impossible for delegation to inherently reduce accuracy. Thus, with almost any delegation mechanism there is strong reason to expect that delegations may improve or, at least, maintain accuracy.

\begin{table}[t]
\centering
\begin{tabular}{llcccc}
\cline{2-6}
   & \multicolumn{5}{c}{m}                           \\ \hline
n  & 11      & 21      & 31      & 41      & 51      \\
11 & 7.7e-08 & 2.6e-14 & 8.9e-21 & 3.0e-27 & 1.0e-33 \\
21 & 3.0e-09 & 5.5e-17 & 9.9e-25 & 1.8e-32 & 3.2e-40 \\
31 & 4.0e-10 & 1.1e-18 & 3.3e-27 & 9.6e-36 & 2.7e-44 \\
41 & 9.2e-11 & 6.9e-20 & 5.2e-29 & 3.9e-38 & 2.9e-47 \\
51 & 2.8e-11 & 7.5e-21 & 1.9e-30 & 5.2e-40 & 1.3e-49 \\ \hline
\end{tabular}
\caption{A loose upper bound on the fraction of states with $n$ voters and a dataset of size $m$ where any single delegation reduces group accuracy.}
\label{tab:harmful_state_upper_bound}
\end{table}

\section{Ensemble Training Cost}
\label{sec:training_cost}

We now discuss how to calculate the training cost (TC) of an ensemble. Denote the number of times the training algorithm of the model underlying $v_i$ makes use of the training data as $e_i$.
Generally, we define the training cost of $v_i$ as $e_i$ multiplied by $m$, the size of the training data. The sum of training costs over all classifiers is the training cost of the ensemble.

\begin{equation}
\label{eq:training_cost}
TC = \sum_{i \in V} m e_i 
\end{equation}

However, during delegation the number of classifiers training is reduced over time, making calculation of TC nontrivial.
Here we describe the calculation of TC for a delegating ensemble and show how parameter values affect it. This formula has two components: the first to capture delegation cost and the second to capture the cost of fully training the final ensemble, $V^\text{final}$.

\begin{equation*}
\textbf{TC} = \text{Delegation Cost} + \sum_{v_i \in V^\text{final}} me_i
\end{equation*}

\noindent
\textbf{Delegation Cost:}
When calculating cost we assume that the dataset is of sufficient size to fully delegate (note that this is not always the case with very small datasets). Here we calculate a lower bound on delegation cost. In practice, issues such as avoiding delegation cycles may slightly increase the number of increments needed.

The cost depends upon increment size $u$ (the number of examples trained upon in each increment), delegation rate $1-r$ ($r$ denotes the fraction of voters that \textit{do not} delegate in an increment), number of increments $z$, initial number of voters $n$, and final number of representatives $n^\text{final}$; it is the geometric series,

\[
\sum_{i=0}^{z} unr^i = un(\frac{1-r^{z+1}}{1-r})
\]

The minimum number of increments can be calculated as a function of $n$, $n^\text{final}$, and $r$ by noting that delegation continues until the number of active voters is equal to $n^\text{final}$.

\[
nr^z = n^\text{final}
\]

Giving,

\[
z = \frac{log(n^\text{final}/n)}{log(r)}
\]

Combining the formulas above results in a function for calculating delegation cost based on $n$, $n^\text{final}$, and $r$. We leave out increment size, $u$, from our calculation as it is a constant factor. \autoref{fig:delegation_cost} shows this lower bound on delegation cost for a range of parameters. Somewhat surprisingly, there is very little difference between differing numbers of final representatives at the same delegation rate. This is due to the fact that a smaller number of representatives has a lower training cost but a higher number of increments.

\begin{figure}[t]
    \centering
    \includegraphics[width=0.7\columnwidth]{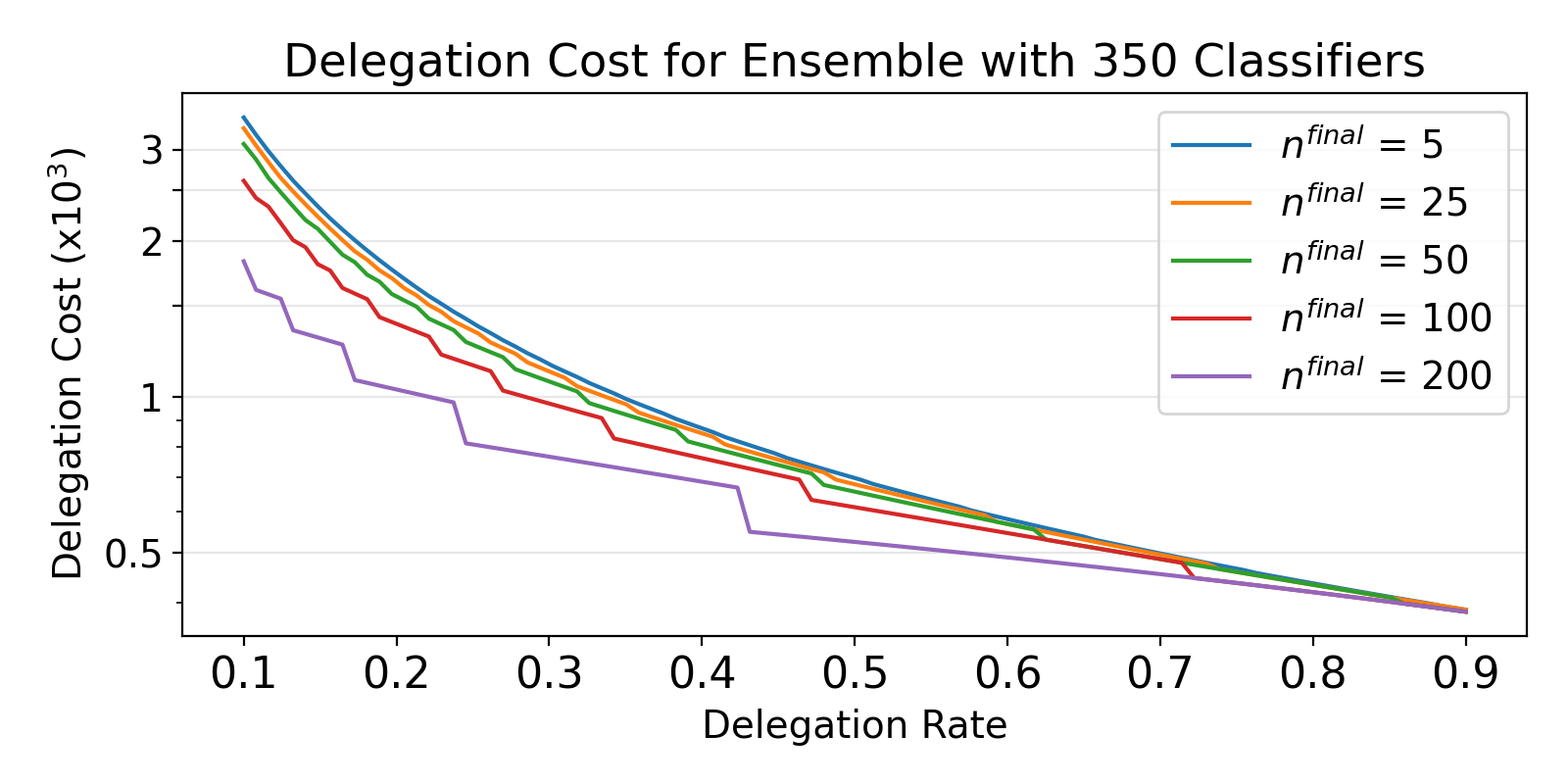}
    \caption{The lower bound on delegation cost as $n^\text{final}$ and delegation rate are varied.}
    \label{fig:delegation_cost}
\end{figure}

\noindent
\textbf{Cost to Train Full Ensemble:}
The cost for fully fitting an ensemble post-delegation varies based upon training data size, number of remaining representatives, and how many iterations of training must be performed with each representatives. We count the number of iterations over the training data taken to train each classifier and multiply by the size of the training set. Summing over all remaining classifiers gives the training cost of the final ensemble.

\section{Experiments}
\label{sec:experiments}

We empirically demonstrate the benefits of our training process. The experiments presented here are implemented using Python 3.8. The classifier we use in our ensembles is Scikit-learn's SGDClassifier with default parameters, varying in their random initializations. Each experiment uses different parameters which are described subsequently although, in each experiment, $n^\text{final} = 10$. We evaluate our model based on several datasets found in the UCI Machine Learning Repository \cite{kelly2023uci}. \autoref{tab:datasets} details the size and feature types of each dataset. All of the datasets studied in this work have two classes.

When evaluating our ensembles we focus upon three primary dimensions: (1) the performance of the ensemble itself, in terms of test accuracy and F1 score, (2) the training cost of each ensemble relative to the cost of fully training an ensemble of $n$ classifiers, and (3) the weight distribution within the ensemble. Weight does not directly impact the performance of the ensemble but high centralization of weight can lead to more sensitivity to outliers and removes many advantages of using ensembles.

\subsection{Training Cost of Comparison Methods}

Our experiments compare training cost of delegated ensembles with the training cost of various parameterizations of Adaboost, a powerful ensemble boosting algorithm \cite{freund1995desicion}. Specifically, we use the implementation of Adaboost provided in sklearn \cite{scikit-learn}.

We calculate the training cost of Adaboost as in \autoref{eq:training_cost} by counting the number of times in which each classifier in an ensemble has seen each example in the training set. Specifically, we sum the sklearn variable \texttt{n\_iter\_} for each classifier within the boosted ensemble. The result is the training cost of the ensemble.

\begin{table}[t]
\centering
\begin{tabular}{@{}lccc@{}}
\toprule
Dataset             & Size  & \# Categorical & \# Numerical \\ \midrule
breast-cancer-w     & 699   & 0              & 9            \\
credit-approval     & 690   & 9              & 6            \\
heart               & 270   & 6              & 7            \\
ionosphere          & 351   & 34             & 0            \\
kr-vs-kp            & 3196  & 0              & 36           \\
online-shopper      & 12330 & 8              & 10           \\
occupancy-detection & 20560 & 0              & 6            \\
spambase            & 4601  & 0              & 57           \\ \bottomrule
\end{tabular}
\caption{Description of datasets used. All have two classes and are found in the UCI Machine Learning Repository \cite{kelly2023uci}.}
\label{tab:datasets}
\end{table}

\subsection{Optimal Parameter Values}
\label{sec:param_search}

Three parameters most directly affect both the accuracy and cost of our training process: delegation rate, initial number of voters, and increment size. We explore a range of values for each of these parameters and present the results in \autoref{fig:experiment-parameter_search}. Each hexagon in these ternary graphs represents average ensemble accuracy on test data after 50 trials with one set of parameters. In each trial, classifiers are created with new random seeds, the dataset shuffled, and new test/train data sampled.

The accuracy corresponds to the colour bar underneath the figure with lighter values indicating higher accuracy. The parameter values aligned with each cell can be found by the direction of tick marks outside the graph. The bottom axis, delegation rate, goes ``up and to the right''. Accordingly, the highest accuracy with the \textit{max} delegation mechanism (top left in \autoref{fig:experiment-parameter_search}) is found when increment size is 65 or 85, delegation rate is 0.05, and ensemble size is 200 or 350.

Across delegation methods this parameter search generally finds that large increment size, a relatively small ensemble, and a low delegation rate lead to the highest accuracy. This strongly indicates a set of parameters to use in subsequent experiments in order to maximize accuracy. However, note that calculations in \autoref{sec:training_cost} show low delegation rate to greatly increase training cost, thus presenting a trade-off between ensemble performance and training cost.

\begin{figure}[t]
     \centering
     \begin{subfigure}[b]{0.45\textwidth}
         \centering
         \includegraphics[width=0.75\textwidth]{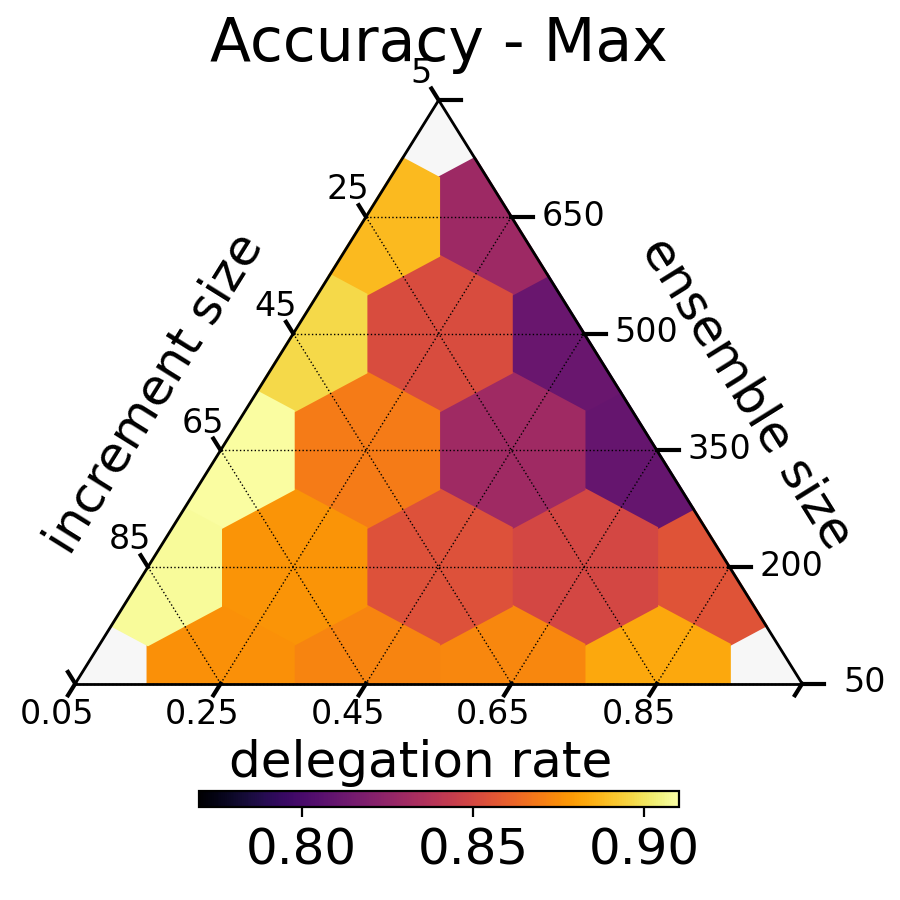}
     \end{subfigure}
     \hfill
     \begin{subfigure}[b]{0.45\textwidth}
         \centering
         \includegraphics[width=0.75\textwidth]{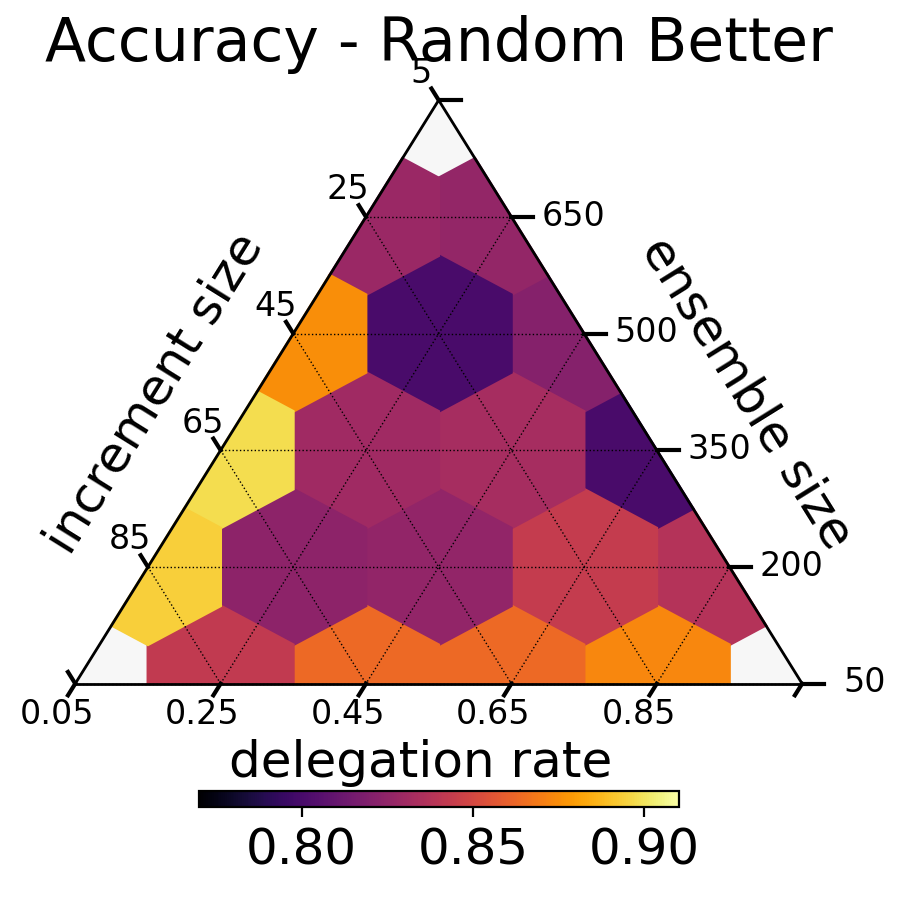}
     \end{subfigure}
     \\ 
     \begin{subfigure}[b]{0.45\textwidth}
         \centering
         \includegraphics[width=0.75\textwidth]{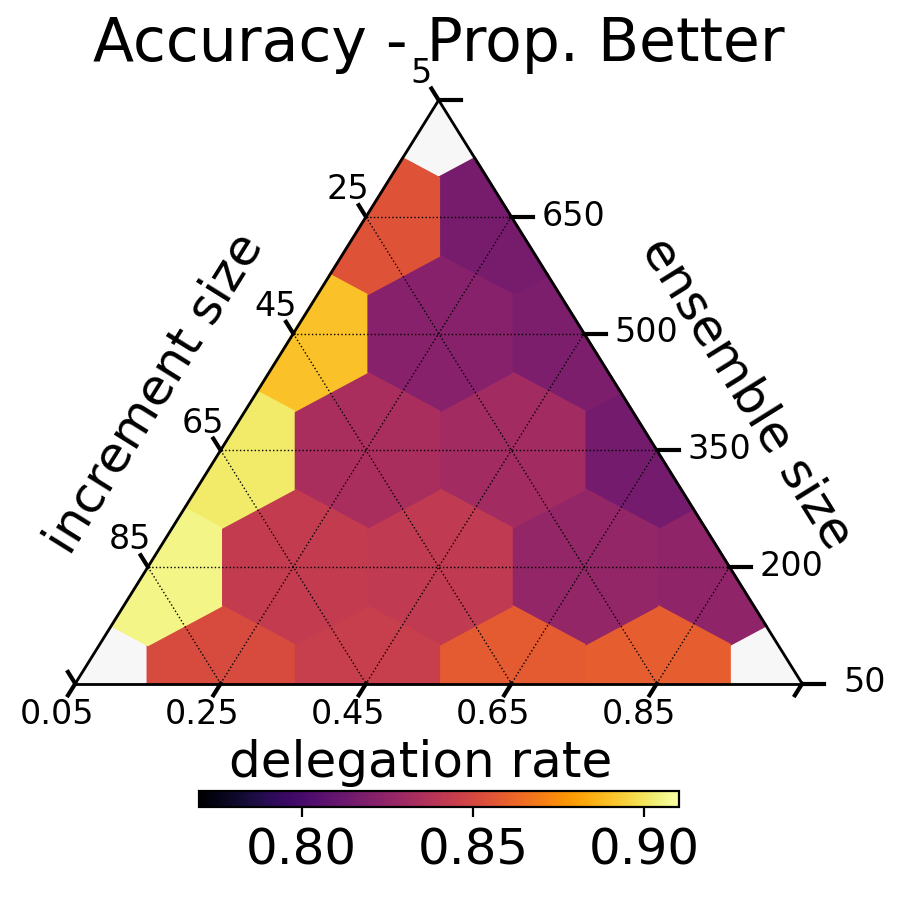}
     \end{subfigure}
     \hfill
     \begin{subfigure}[b]{0.45\textwidth}
         \centering
         \includegraphics[width=0.75\textwidth]{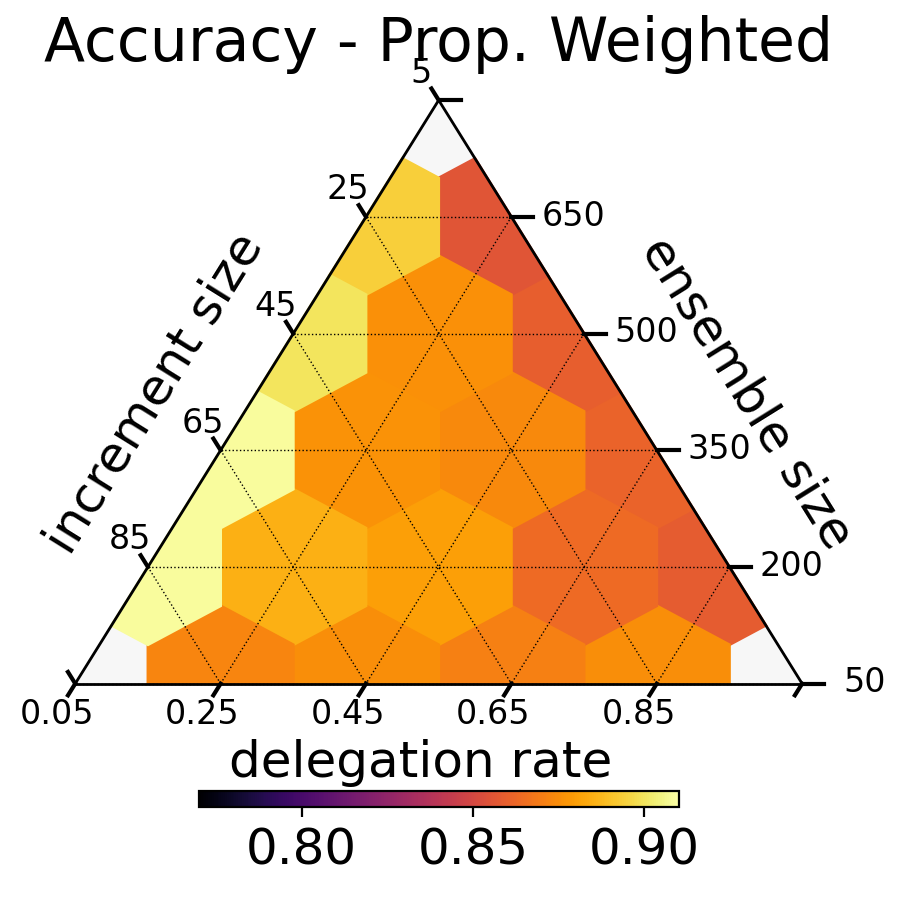}
     \end{subfigure}
     \caption{Test accuracy of fully trained ensemble across delegation methods as parameters affecting accuracy are varied. Results displayed are from the spambase dataset. Random delegations are omitted as they perform significantly worse than the displayed delegation mechanisms; Direct delegations are omitted as increment size and delegate rate do not affect Direct ensemble performance.}
     \label{fig:experiment-parameter_search}
\end{figure}

\subsection{Ensemble Accuracy During Delegation}
\label{sec:ensemble_performance}

In order to better illustrate the dynamics of delegation, \autoref{fig:spambase_accuracy_over_training} shows an experiment with parameters designed to lead to many delegation steps. Here we present results of an experiment averaged over 500 trials on the spambase dataset with increment size of 25, delegation rate of 0.2, and ensemble size of 350. \autoref{appendix:results_all_datasets} shows the same experiment for all other datasets.

\autoref{fig:spambase_accuracy_over_training} (left) plots the test accuracy of the partially trained ensemble at each time step during delegation. The y-axis shows accuracy of the ensemble at the current time step (accuracy on the entire test set; the same at each step). The x-axis shows the number of representatives which corresponds exactly with time step.
On the Kolmogorov-Smirnov and Mann-Whitney U tests, the difference between final test accuracy of each delegation method is statistically significant ($p < 0.01$). While Proportional Weighted delegations give the highest accuracy, they are only slightly more accurate than Max delegations. While all delegation methods (aside from the baselines of Direct and Random) quickly plateau at a similar accuracy, Proportional Weighted is able to undergo significantly more delegation without a reduction in accuracy. This leads to a final ensemble with much lower inference cost than if delegation had stopped at the time step when $75$ classifiers remained as representatives; where other delegation methods had peak accuracy.

The advantage of Proportional Weighted is further demonstrated by considering the final weight distribution across classifiers in the ensemble. For the same experiment, \autoref{fig:spambase_accuracy_over_training} (right) shows that throughout training Proportional Weighted consistently requires a higher number of classifiers in agreement with each other in order to make a classification decision. This shows a much lower centralization of weight resulting from Proportional Weighted which leads to fewer dictatorships. When one, or very few, classifiers hold the majority of weight then classification decisions can become much less stable and accuracy decreases.



\begin{figure}[ht!]
     \centering
     \begin{subfigure}[b]{0.49\textwidth}
         \centering
        \includegraphics[width=0.98\columnwidth, keepaspectratio]{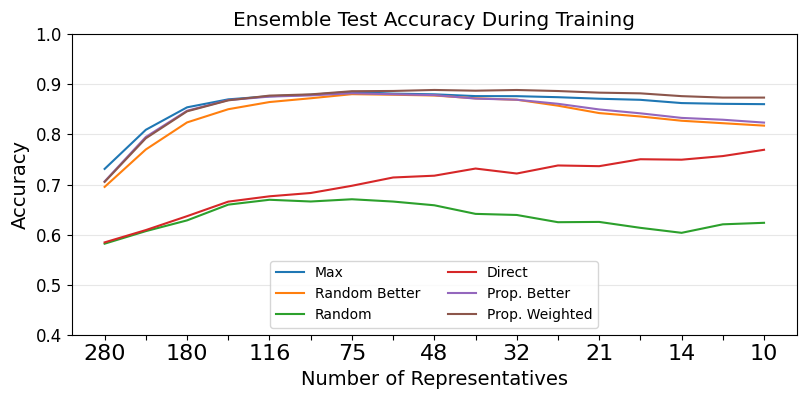}
     \end{subfigure}
     \hfill
     \begin{subfigure}[b]{0.49\textwidth}
         \centering
        \includegraphics[width=0.98\columnwidth, keepaspectratio]{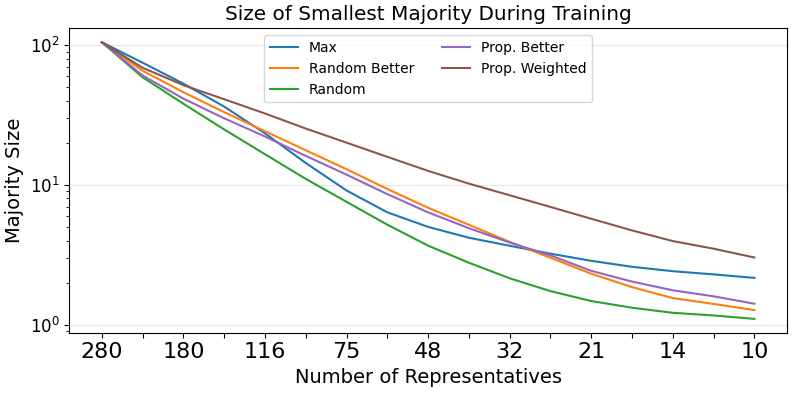}
     \end{subfigure}
     \caption{\textbf{(left)} Test accuracy during training on spambase dataset, averaged over 500 trials. \textbf{(right)} Minimum majority size during training on the spambase dataset.}
    \label{fig:spambase_accuracy_over_training}
\end{figure}

\subsection{Comparison with Other Ensemble Methods}
\label{sec:comparison_with_adaboost}

\begin{table*}[t]
\centering
\scriptsize
\begin{tabular}{lcccccccccccc}
\hline
Ensemble      & \multicolumn{3}{c}{breast-cancer-w}                                              & \multicolumn{3}{c}{credit-approval}        & \multicolumn{3}{c}{heart}                           & \multicolumn{3}{c}{ionosphere}          \\ \hline
              & \multicolumn{1}{l}{Acc}   & F1             & \multicolumn{1}{c|}{Cost}           & Acc   & F1    & \multicolumn{1}{c|}{Cost}  & Acc   & F1             & \multicolumn{1}{c|}{Cost}  & Acc            & F1             & Cost  \\
Direct        & 0.907                     & 0.926          & \multicolumn{1}{c|}{1}              & 0.628 & 0.585 & \multicolumn{1}{c|}{1}     & 0.584 & 0.59           & \multicolumn{1}{c|}{1}     & 0.854          & 0.765          & 1     \\
Prop W Acc    & 0.907                     & 0.927          & \multicolumn{1}{c|}{0.782}          & 0.631 & 0.586 & \multicolumn{1}{c|}{0.817} & 0.573 & 0.574          & \multicolumn{1}{c|}{0.9}   & 0.853          & 0.762          & 0.894 \\
Prop W Cost   & 0.9                       & 0.92           & \multicolumn{1}{c|}{\textbf{0.033}} & 0.609 & 0.584 & \multicolumn{1}{c|}{0.036} & 0.565 & 0.547          & \multicolumn{1}{c|}{0.035} & 0.802          & 0.716          & 0.034 \\
Ada DT Full   & 0.953                     & 0.932          & \multicolumn{1}{c|}{0.039}          & 0.818 & 0.833 & \multicolumn{1}{c|}{0.07}  & 0.758 & 0.783          & \multicolumn{1}{c|}{0.027} & 0.916          & 0.937          & 0.031 \\
Ada DT Small  & \multicolumn{1}{l}{0.957} & 0.938          & \multicolumn{1}{c|}{0.001}          & 0.852 & 0.862 & \multicolumn{1}{c|}{0.002} & 0.803 & 0.824          & \multicolumn{1}{c|}{0.001} & 0.896          & 0.921          & 0.001 \\
Ada SGD Full  & \multicolumn{1}{l}{0.965} & 0.95           & \multicolumn{1}{c|}{0.014}          & 0.653 & 0.681 & \multicolumn{1}{c|}{0.007} & 0.683 & 0.718          & \multicolumn{1}{c|}{0.01}  & 0.861          & 0.898          & 0.025 \\
Ada SGD Small & \multicolumn{1}{l}{0.965} & 0.95           & \multicolumn{1}{c|}{0.013}          & 0.65  & 0.674 & \multicolumn{1}{c|}{0.007} & 0.678 & 0.718          & \multicolumn{1}{c|}{0.009} & 0.861          & 0.898          & 0.015 \\ \hline
Ensemble      & \multicolumn{3}{c}{kr-vs-kp}                                                     & \multicolumn{3}{c}{occupancy-det}          & \multicolumn{3}{c}{online-shoppers}                 & \multicolumn{3}{c}{spambase}            \\ \hline
              & \multicolumn{1}{l}{Acc}   & F1             & \multicolumn{1}{c|}{Cost}           & Acc   & F1    & \multicolumn{1}{c|}{Cost}  & Acc   & F1             & \multicolumn{1}{c|}{Cost}  & Acc            & F1             & Cost  \\
Direct        & 0.91                      & 0.903          & \multicolumn{1}{c|}{1}              & 0.946 & 0.964 & \multicolumn{1}{c|}{1}     & 0.869 & 0.927          & \multicolumn{1}{c|}{1}     & 0.86           & 0.88           & 1     \\
Prop W Acc    & \textbf{0.947}            & \textbf{0.943} & \multicolumn{1}{c|}{0.269}          & 0.94  & 0.96  & \multicolumn{1}{c|}{0.055} & 0.843 & \textbf{0.906} & \multicolumn{1}{c|}{0.058} & \textbf{0.909} & \textbf{0.927} & 0.198 \\
Prop W Cost   & 0.908                     & 0.902          & \multicolumn{1}{c|}{0.026}          & 0.916 & 0.936 & \multicolumn{1}{c|}{0.029} & 0.768 & \textbf{0.817} & \multicolumn{1}{c|}{0.029} & \textbf{0.869} & \textbf{0.89}  & 0.029 \\
Ada DT Full   & 0.966                     & 0.968          & \multicolumn{1}{c|}{0.02}           & 0.99  & 0.978 & \multicolumn{1}{c|}{0.009} & 0.888 & 0.605          & \multicolumn{1}{c|}{0.019} & 0.934          & 0.916          & 0.026 \\
Ada DT Small  & \multicolumn{1}{l}{0.946} & 0.948          & \multicolumn{1}{c|}{0.001}          & 0.989 & 0.977 & \multicolumn{1}{c|}{0}     & 0.89  & 0.62           & \multicolumn{1}{c|}{0.001} & 0.916          & 0.891          & 0.001 \\
Ada SGD Full  & \multicolumn{1}{l}{0.941} & 0.944          & \multicolumn{1}{c|}{0.06}           & 0.984 & 0.966 & \multicolumn{1}{c|}{0.005} & 0.878 & 0.447          & \multicolumn{1}{c|}{0.012} & 0.786          & 0.719          & 0.013 \\
Ada SGD Small & \multicolumn{1}{l}{0.91}  & 0.915          & \multicolumn{1}{c|}{0.01}           & 0.984 & 0.966 & \multicolumn{1}{c|}{0.005} & 0.879 & 0.444          & \multicolumn{1}{c|}{0.011} & 0.791          & 0.742          & 0.012 \\ \hline
\end{tabular}
\caption{
Accuracy, F1 score, and Training Cost (relative to Direct ensembles) when comparing a variety of Adaboost methods against two parameterizations of a delegating ensemble. Prop W Acc uses parameters aimed at increasing the accuracy of the ensemble while Prop W Cost uses parameters intended to reduce training cost. \textit{NOTE}: Bold values indicate when a delegating ensemble outperforms \textit{at least one} Adaboost method.
}
\label{tab:acc_cost_comparison}
\end{table*}

Our final experiment uses parameters chosen based on the results of the parameter sweep in \autoref{sec:param_search}. We ran two sets of experiments -- one aimed at maximizing accuracy and another aimed at reducing training cost. The accuracy maximization experiments used an increment size of $65$ and a delegation rate of $0.05$ while the cost minimization experiments used an increment size of $25$ and a delegation rate of $0.85$. All experiments used an ensemble of $350$ classifiers and we average the results over 50 trials on each dataset.
On most datasets Proportional Weighted generally (but not always) resulted in the best performance of delegating ensembles; for simplicity we present results of Proportional Weighted here and show results for all delegation methods in \autoref{appendix:complete_acc_and_cost_results}.

We compare the above experiments with direct ensembles making no delegations and four versions of Adaboost varied along two parameters: DT/SGD and Full/Small. Ada DT uses a decision tree as the underlying classifier (using sci-kit learn's default parameters) while Ada SGD uses an SVM (with default parameters). Ada Full could use up to $350$ classifiers (the size of the initial ensemble) while Ada Small used up to $10$ classifiers (the amount in the fully delegated ensembles).

In \autoref{tab:acc_cost_comparison} we show the accuracy, F1-score and training cost relative to Direct ensembles (i.e. the case of no delegation). \textit{Prop W Acc} refers to results from Proportional Weighted delegations using the accuracy-maximizing parameters while \textit{Prop W Cost} refers to results from Proportional Weighted delegations using the cost-minimizing parameters.
Note that we bold only Prop W Acc and Prop W Cost; bold values indicate beating at least one of the Adaboost methods.

Proportional Weighted is able to dramatically reduce training cost from the full ensemble and while the difference between experiments optimizing for accuracy and those optimizing for cost is clear there is only a mild increase in accuracy/F1 Score for a significant increase in cost.
When comparing to Adaboost we see that accuracy and cost are similar in many instances and on larger datasets (bottom row) Prop W Acc outperforms Adaboost in many cases.

\section{Discussion}

This paper proposes a novel application of liquid democracy to ensemble learning. Using delegations with incremental training of ensembles we are able to dramatically reduce the training cost of ensembles while improving accuracy over a full ensemble. We explore a variety of delegation procedures and show that the Proportional Weighted mechanism outperforms other mechanisms on both accuracy and training cost.
Our procedure is able to provide higher accuracy than multiple forms of Adaboost with underlying SGD or Decision Tree classifiers.
The parameterization of our algorithm allows for direct management of the trade-off between training cost and accuracy.
Delegation lends itself naturally to a method of ensemble pruning that scales very well with the initial number of classifiers and reliably produces a high-accuracy ensemble with low training cost. 

Our results open future work dedicated to identifying more powerful delegation mechanisms, theoretical analysis of delegation mechanism quality, and a study of alternative delegation schedules - in this work delegation always proceeds at a constant rate but that need not always hold true.
Delegation may also provide an intuitive model for the type of multi-domain classification studied by \cite{zhou2021domain} by using delegation to quickly transfer weight to domain experts within an ensemble as domains shift.

Our delegative training procedure can be immediately applied for use by existing ensembles. It is, however, conceptually well suited for two particular application areas. For very large datasets (or when very little memory is available) it is possible to skip the final step of fully training the ensemble and use this method for out-of-core learning, where the entire training dataset cannot all fit into memory \cite{Saturn_Cloud_2023}. Our method could also be applied to online learning by training on data in batches as it becomes available.

A notable practical limitation of this procedure is that the incremental training procedure we rely upon trains each classifier on new data without forgetting previously learned data. This is well supported by many model types such as SVMs and neural networks, but training algorithms to use incremental training with some models, such as Decision Trees, are less widely used (and are not guaranteed to exist for every model type).

\printbibliography


\newpage
\clearpage

\appendix

\section{Chance of Weak Improvement from Delegation}
\label{appendix:weak_improvement}

\begin{theorem}

    In an ensemble with $n$ voters and $m$ examples, the total number of ways in which classification decisions can be made on pivotal examples such that every voter is pivotal is $s_{n, m}^\text{pivotal} = \sum_{m_p=2}^{m} \binom{n}{\ceil{\frac{n}{2}}}^{m_p}$. Without any restrictions the same examples could have 
    $s_{n, m}^\text{total} = \sum_{m_p=2}^{m} 2^{nm_p}$ possible states.
\end{theorem}

\begin{proof}

During this proof we consider each possible classification outcome as a separate state, see \autoref{fig:pivotal_voter_diagrams} for a visualization of 2 states; this can be thought of as a matrix with $n$ rows and $m$ columns where each cell can take only binary values. The cell at position $(i, j)$ refers to whether or not voter $i$ classified example $j$ correctly. We refer to voter predictions, and the predictions made on an example as rows and columns respectively.

Say there are $m_p$ pivotal examples in some initial state. In order for each voter to be pivotal at least once, $2 \leq m_p \leq m$. We can obtain an upper-bound estimate on the fraction of states in which delegation is harmful by counting the number of possible states where every voter is pivotal (\autoref{lemma:pivotal_voters} shows that every harmful state meets this condition) and comparing it to the total number of possible outcomes.
We count, for some $m_p$, the number of ways in which $m_p$ columns on $n$ rows can be arranged such that all voters are pivotal. Denote this $s_{n, m_p}^\text{pivotal}$ and compare it with the total number of ways to arrange those $m_p$ columns, denoted $s_{n, m_p}^\text{total}$.

In practice, we calculate only the ratio of $s_{n, m_p}^\text{pivotal}$ and $s_{n, m_p}^\text{total}$. The $m-m_p$ columns that are not pivotal have the same number of states in each case so we exclude them from our calculation.

The number of ways to construct a single pivotal column of $n$ rows is $\binom{n}{\ceil{\frac{n}{2}}}$. Extended to $m_p$ columns we get $\binom{n}{\ceil{\frac{n}{2}}}^{m_p}$.
Summing over all possible values of $m_p$ we arrive at a loose upper bound on the total number of possible states where every voter is pivotal on $m$ examples and $n$ voters:

\[
s_{n, m}^\text{pivotal} = \sum_{m_p=2}^{m} \binom{n}{\ceil{\frac{n}{2}}}^{m_p}
\]

Whereas, the number of ways to fill in the same columns with no regard for whether or not they are pivotal is simply the number of possible states of an $n \times m_p$ binary matrix, or $2^{nm_p}$. Which, summed over all values of $m_p$ becomes,

\[
s_{n, m}^\text{total} = \sum_{m_p=2}^{m} 2^{nm_p}
\]

\end{proof}

\section{Results for All Datasets}
\label{appendix:results_all_datasets}

Below we show for each dataset the figures presented for the spambase dataset in the main text: (1) accuracy of each delegation mechanisms during the parameter sweep, (2) test accuracy during training of each delegation mechanism, and (3) the minimum number of voters required to form a majority of weight during training. Due to its size, we show only the parameter sweep results for the occupancy-detection dataset as they are averaged over 50 trials while the test accuracy and minimum majority size are averaged over 500 trials.

\newpage
\newpage
\subsection{breast-cancer-wisconsin}

\begin{figure}[ht!]
     \centering
     \begin{subfigure}[b]{0.45\textwidth}
         \centering
         \includegraphics[width=0.75\textwidth]{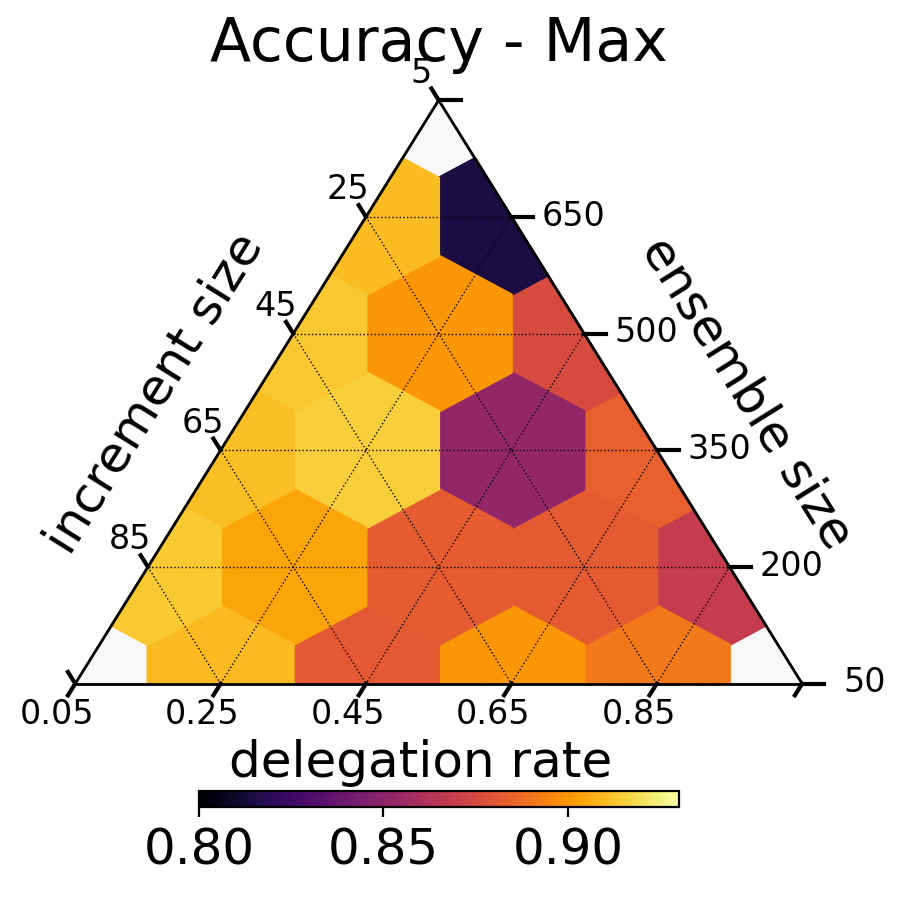}
     \end{subfigure}
     \hfill
     \begin{subfigure}[b]{0.45\textwidth}
         \centering
         \includegraphics[width=0.75\textwidth]{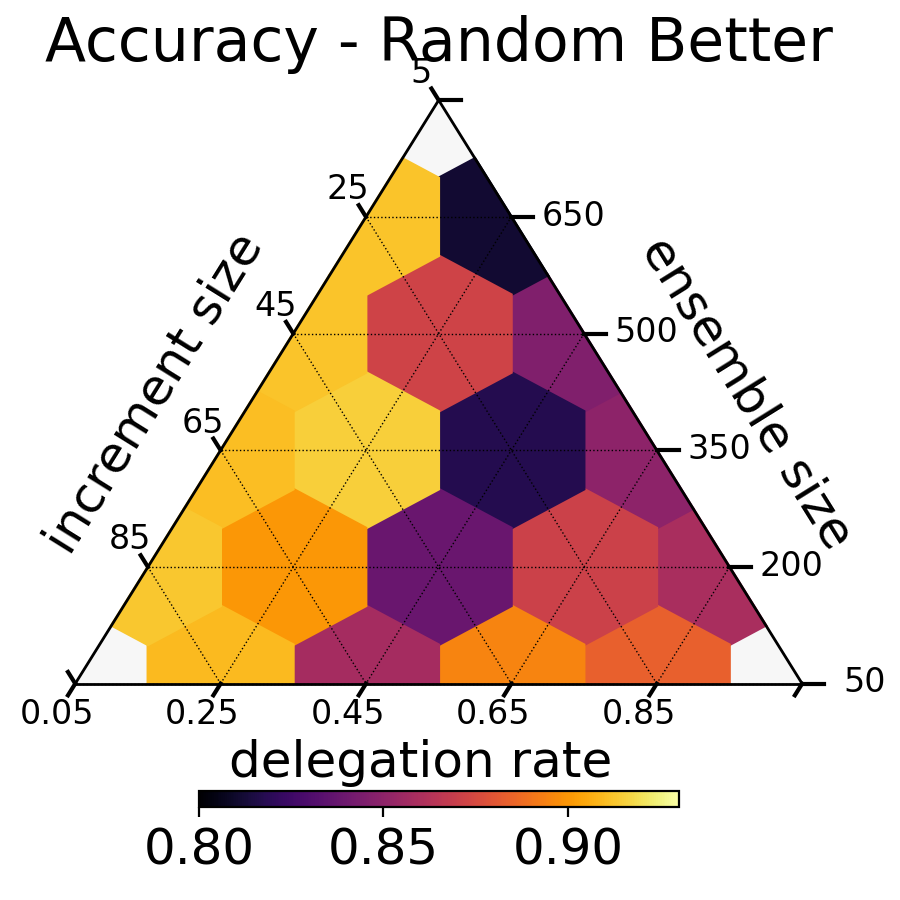}
     \end{subfigure}
     \\ 
     \begin{subfigure}[b]{0.45\textwidth}
         \centering
         \includegraphics[width=0.75\textwidth]{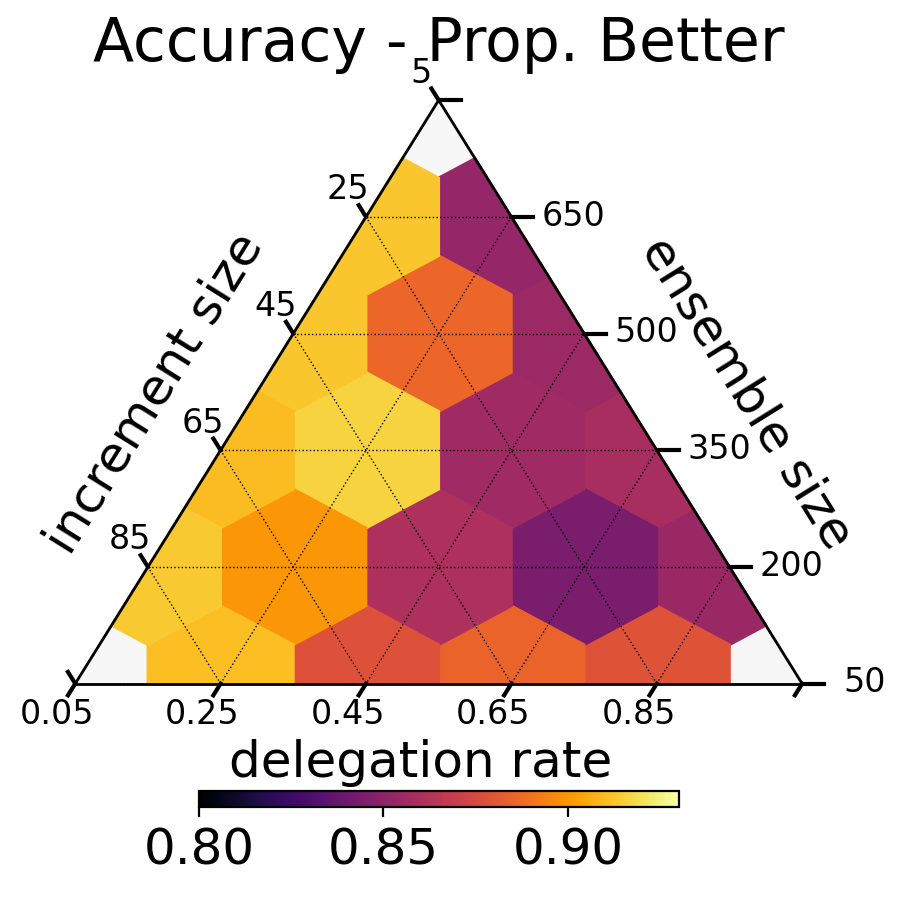}
     \end{subfigure}
     \hfill
     \begin{subfigure}[b]{0.45\textwidth}
         \centering
         \includegraphics[width=0.75\textwidth]{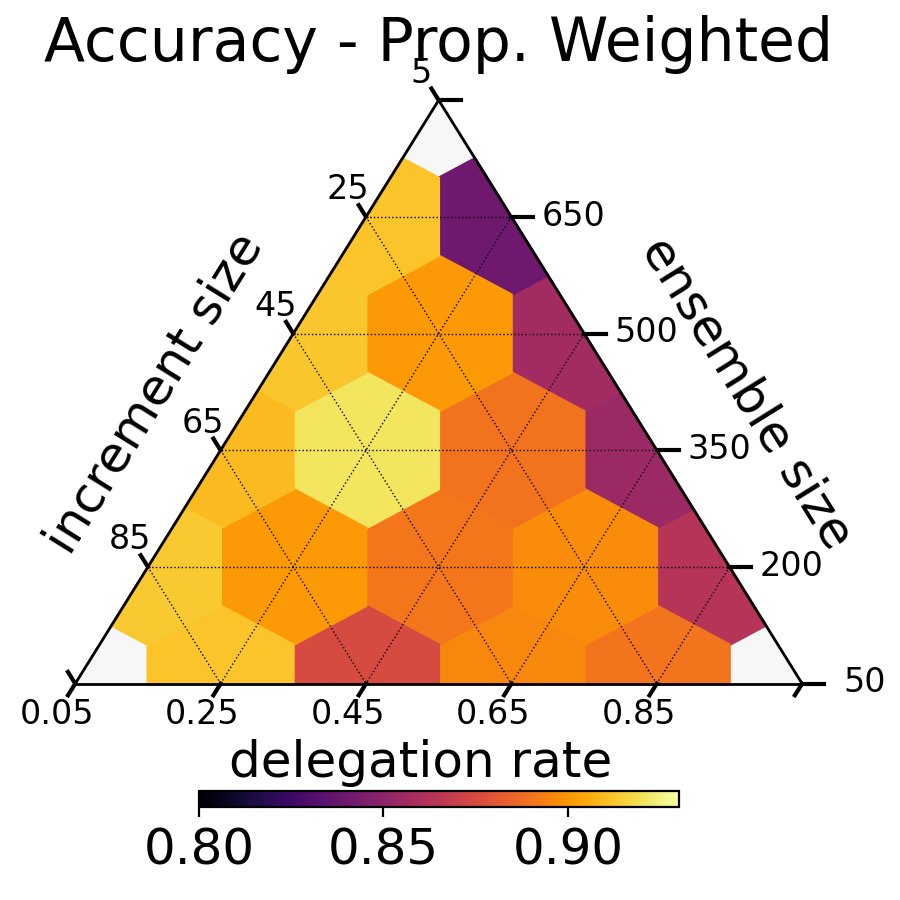}
     \end{subfigure}
     \caption{Test accuracy of fully trained ensembles as parameters varied. Results from breast-cancer-wisconsin dataset.}
\end{figure}

\begin{figure}[ht!]
     \centering
     \begin{subfigure}[b]{0.49\textwidth}
         \centering
        \includegraphics[width=0.98\columnwidth, keepaspectratio]{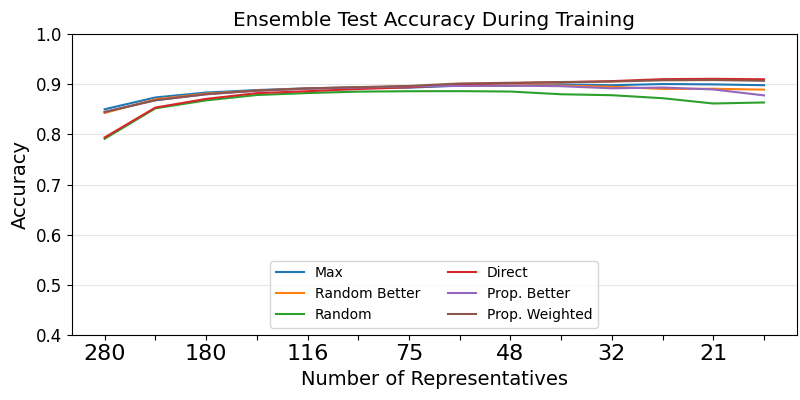}
     \end{subfigure}
     \hfill
     \begin{subfigure}[b]{0.49\textwidth}
         \centering
        \includegraphics[width=0.98\columnwidth, keepaspectratio]{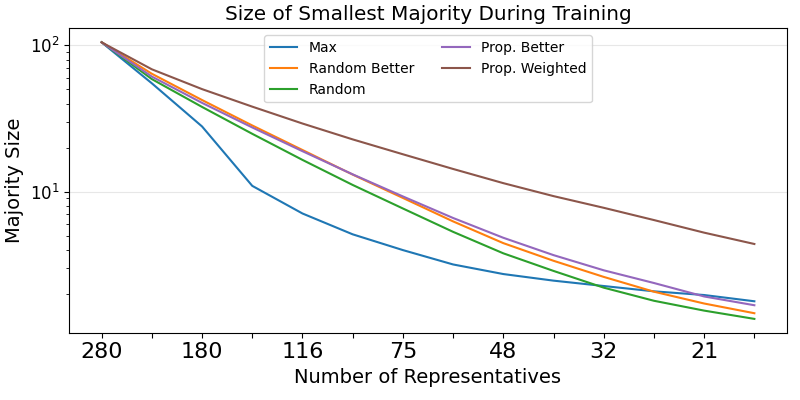}
     \end{subfigure}
     \caption{\textbf{(left)} Test accuracy during training on breast-cancer-wisconsin dataset, averaged over 500 trials. \textbf{(right)} Minimum majority size during training on the breast-cancer-wisconsin dataset.}
\end{figure}


\newpage

\subsection{credit-approval}

\begin{figure}[ht!]
     \centering
     \begin{subfigure}[b]{0.45\textwidth}
         \centering
         \includegraphics[width=0.75\textwidth]{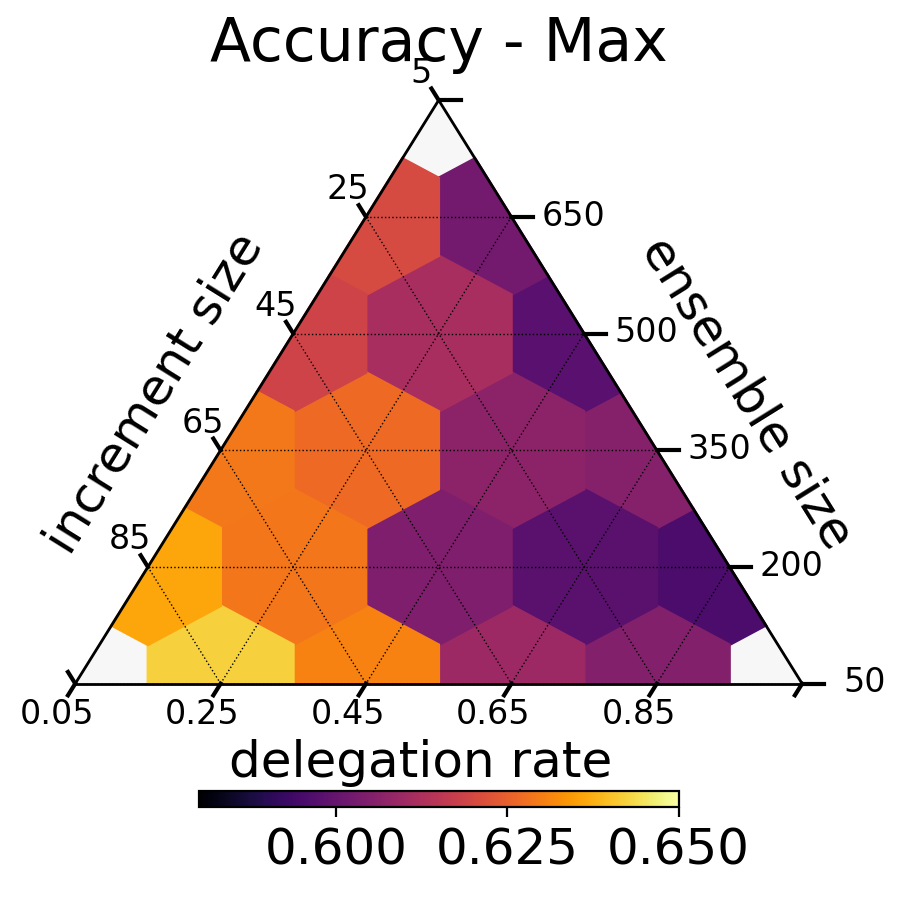}
     \end{subfigure}
     \hfill
     \begin{subfigure}[b]{0.45\textwidth}
         \centering
         \includegraphics[width=0.75\textwidth]{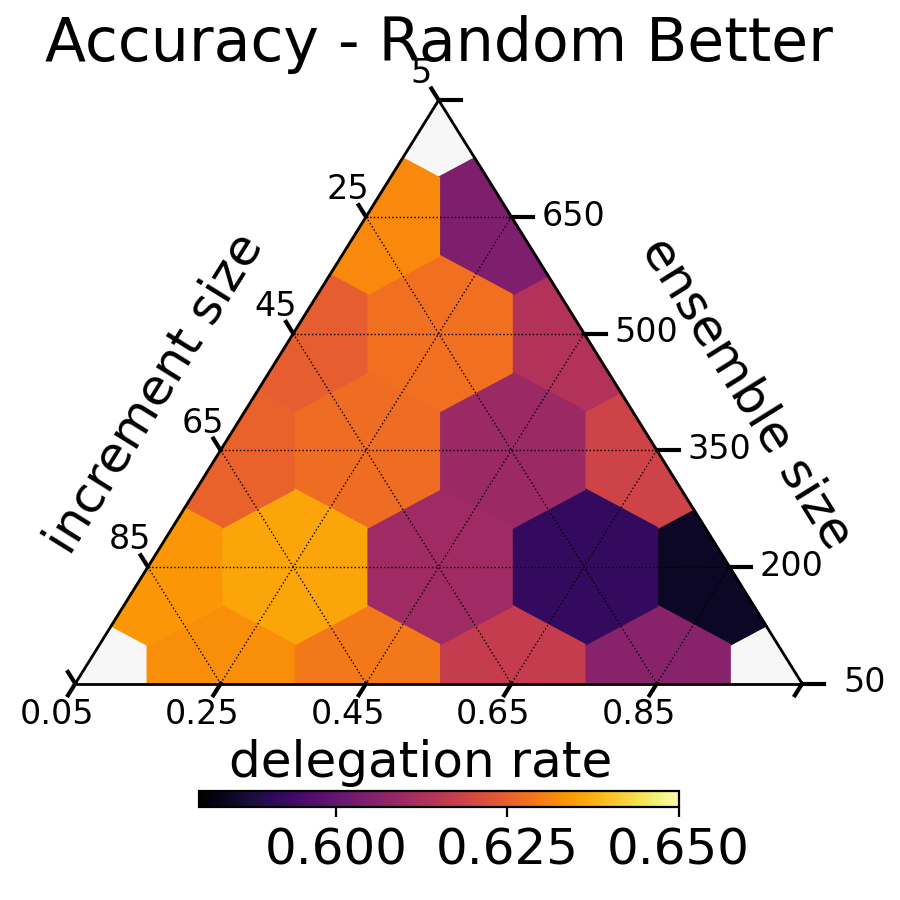}
     \end{subfigure}
     \\ 
     \begin{subfigure}[b]{0.45\textwidth}
         \centering
         \includegraphics[width=0.75\textwidth]{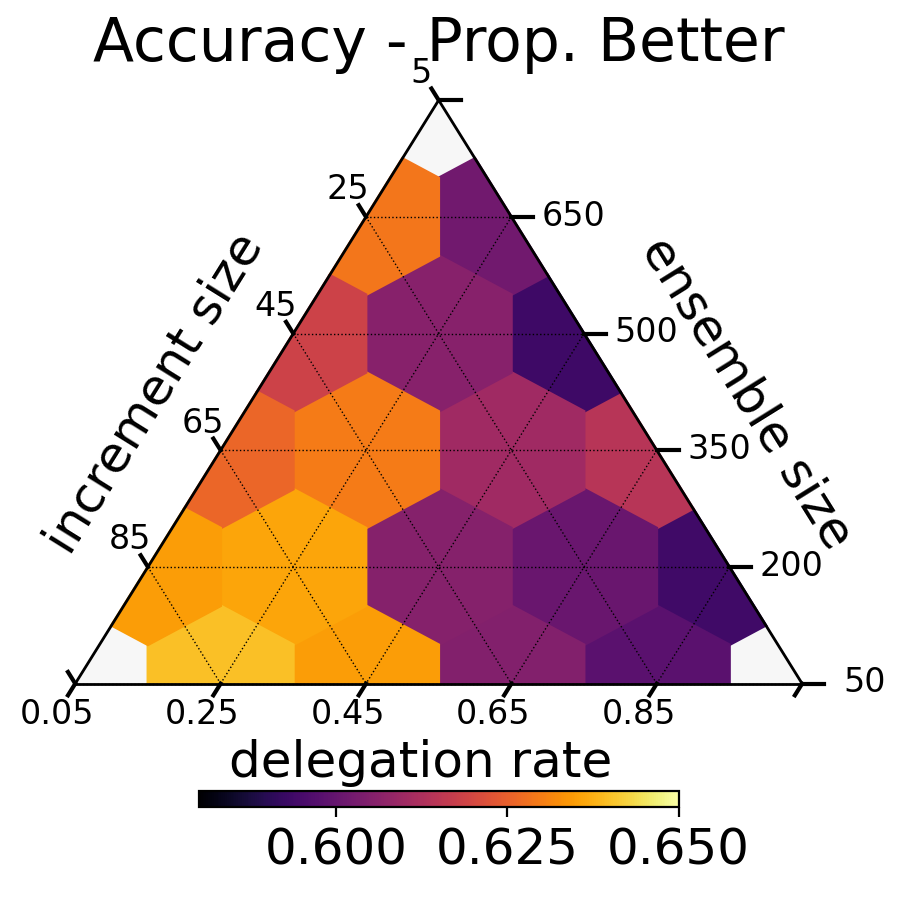}
     \end{subfigure}
     \hfill
     \begin{subfigure}[b]{0.45\textwidth}
         \centering
         \includegraphics[width=0.75\textwidth]{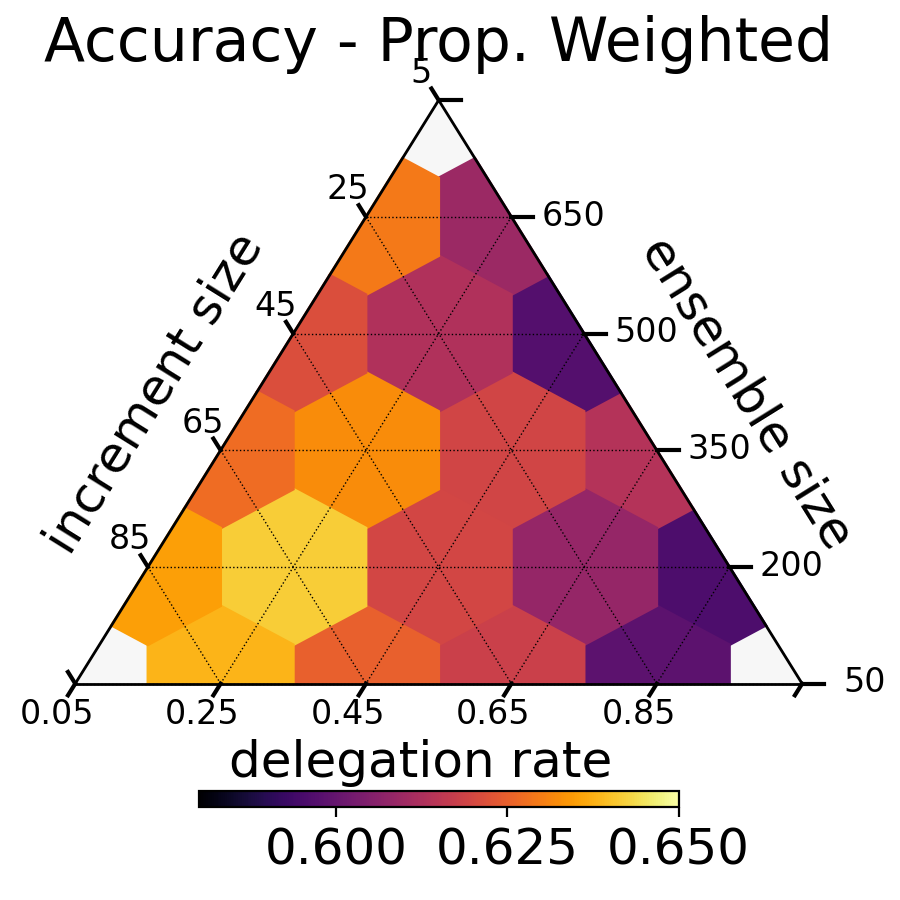}
     \end{subfigure}
     \caption{Test accuracy of fully trained ensembles as parameters varied. Results from credit-approval dataset.}
\end{figure}

\begin{figure}[ht!]
     \centering
     \begin{subfigure}[b]{0.49\textwidth}
         \centering
        \includegraphics[width=0.98\columnwidth, keepaspectratio]{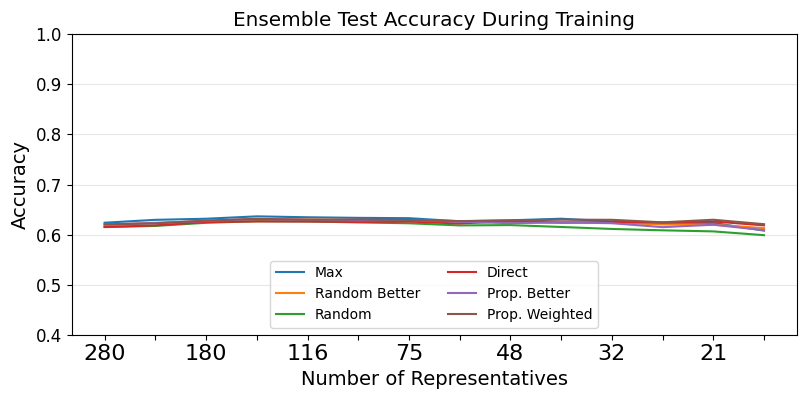}
     \end{subfigure}
     \hfill
     \begin{subfigure}[b]{0.49\textwidth}
         \centering
        \includegraphics[width=0.98\columnwidth, keepaspectratio]{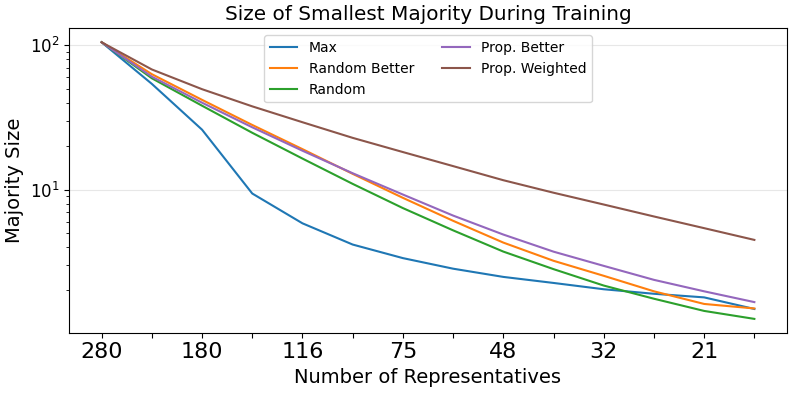}
     \end{subfigure}
     \caption{\textbf{(left)} Test accuracy during training on credit-approval dataset, averaged over 500 trials. \textbf{(right)} Minimum majority size during training on the credit-approval dataset.}
\end{figure}


\newpage

\subsection{heart}

\begin{figure}[ht!]
     \centering
     \begin{subfigure}[b]{0.45\textwidth}
         \centering
         \includegraphics[width=0.75\textwidth]{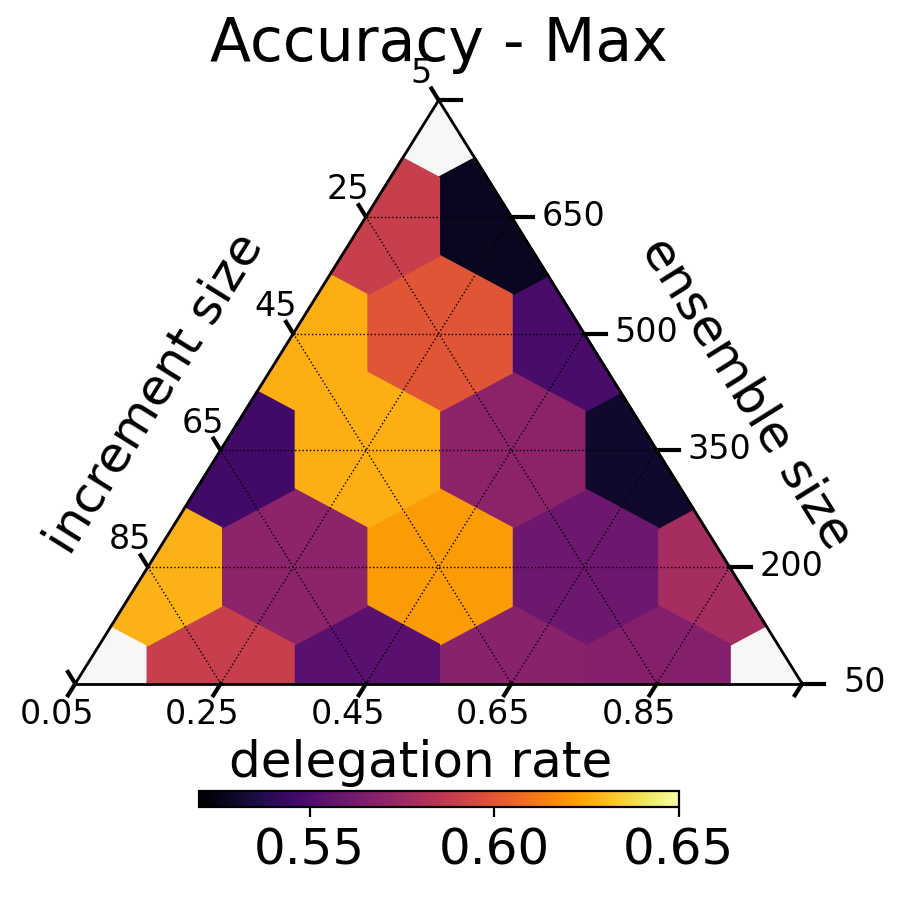}
     \end{subfigure}
     \hfill
     \begin{subfigure}[b]{0.45\textwidth}
         \centering
         \includegraphics[width=0.75\textwidth]{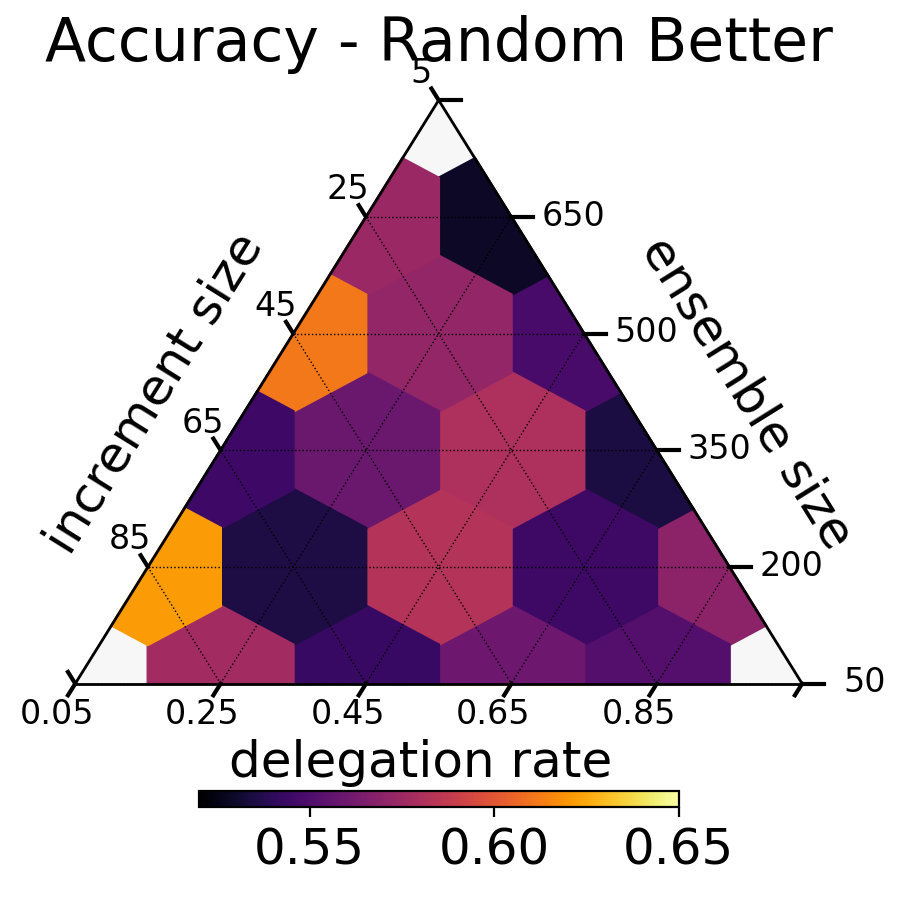}
     \end{subfigure}
     \\ 
     \begin{subfigure}[b]{0.45\textwidth}
         \centering
         \includegraphics[width=0.75\textwidth]{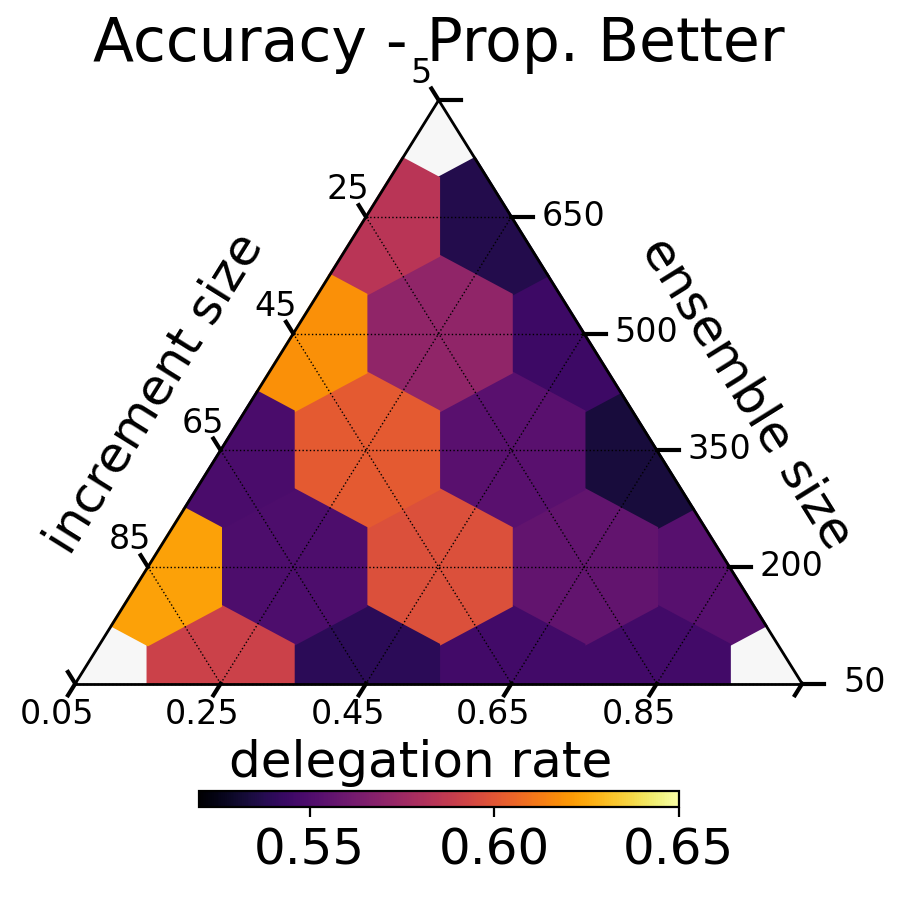}
     \end{subfigure}
     \hfill
     \begin{subfigure}[b]{0.45\textwidth}
         \centering
         \includegraphics[width=0.75\textwidth]{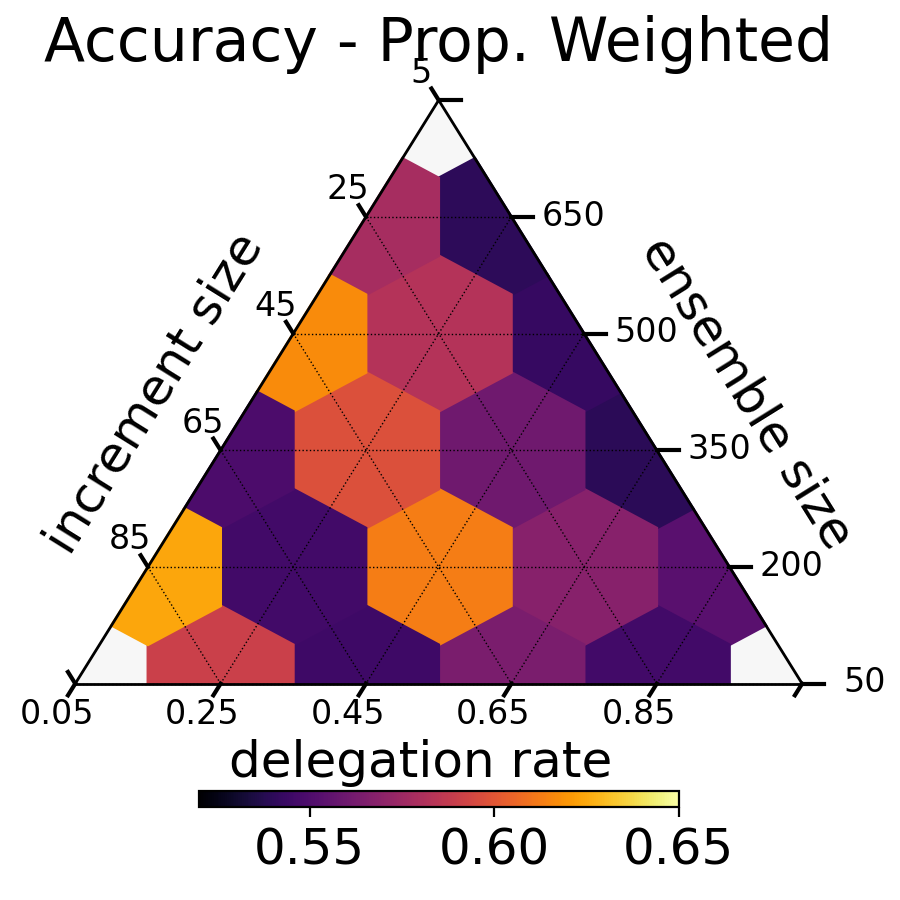}
     \end{subfigure}
     \caption{Test accuracy of fully trained ensembles as parameters varied. Results from heart dataset.}
\end{figure}

\begin{figure}[ht!]
     \centering
     \begin{subfigure}[b]{0.49\textwidth}
         \centering
        \includegraphics[width=0.98\columnwidth, keepaspectratio]{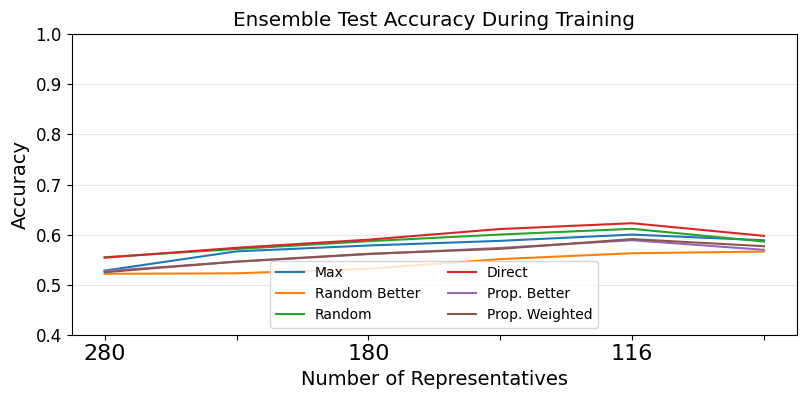}
     \end{subfigure}
     \hfill
     \begin{subfigure}[b]{0.49\textwidth}
         \centering
        \includegraphics[width=0.98\columnwidth, keepaspectratio]{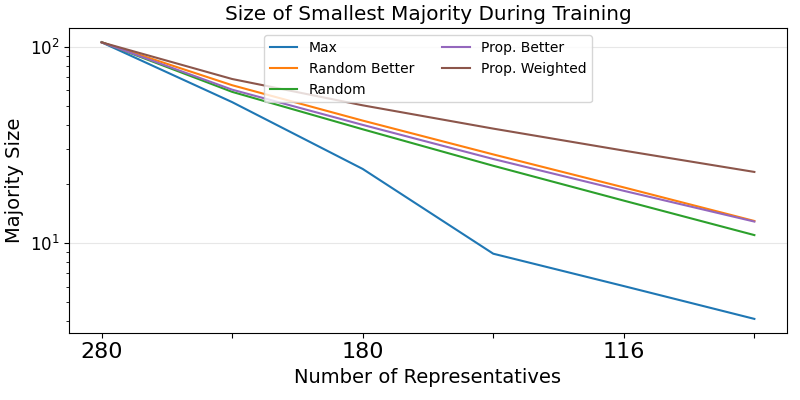}
     \end{subfigure}
     \caption{\textbf{(left)} Test accuracy during training on heart dataset, averaged over 500 trials. \textbf{(right)} Minimum majority size during training on the heart dataset.}
\end{figure}

\newpage

\subsection{ionosphere}

\begin{figure}[ht!]
     \centering
     \begin{subfigure}[b]{0.45\textwidth}
         \centering
         \includegraphics[width=0.75\textwidth]{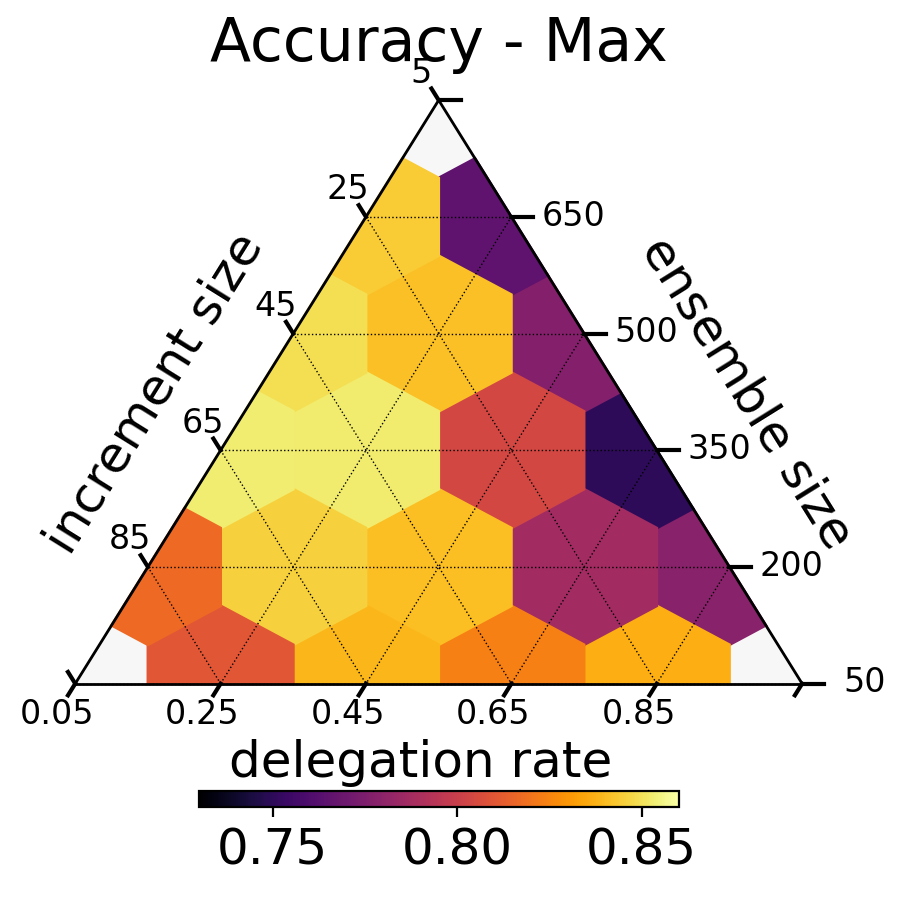}
     \end{subfigure}
     \hfill
     \begin{subfigure}[b]{0.45\textwidth}
         \centering
         \includegraphics[width=0.75\textwidth]{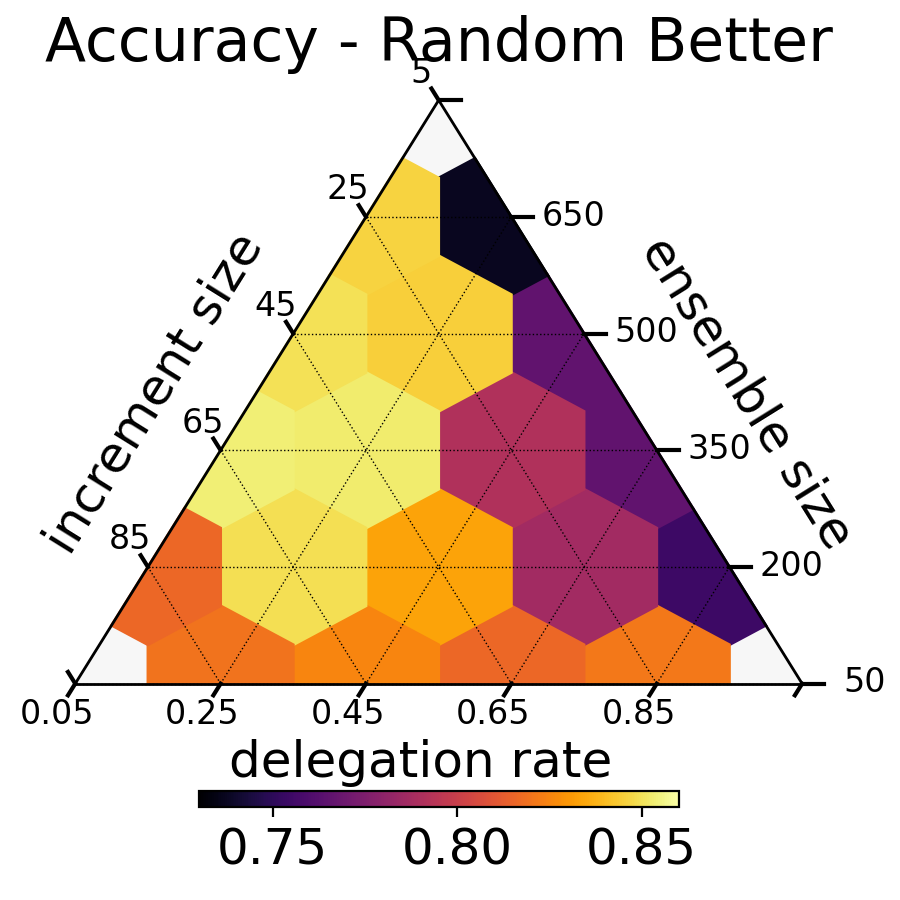}
     \end{subfigure}
     \\ 
     \begin{subfigure}[b]{0.45\textwidth}
         \centering
         \includegraphics[width=0.75\textwidth]{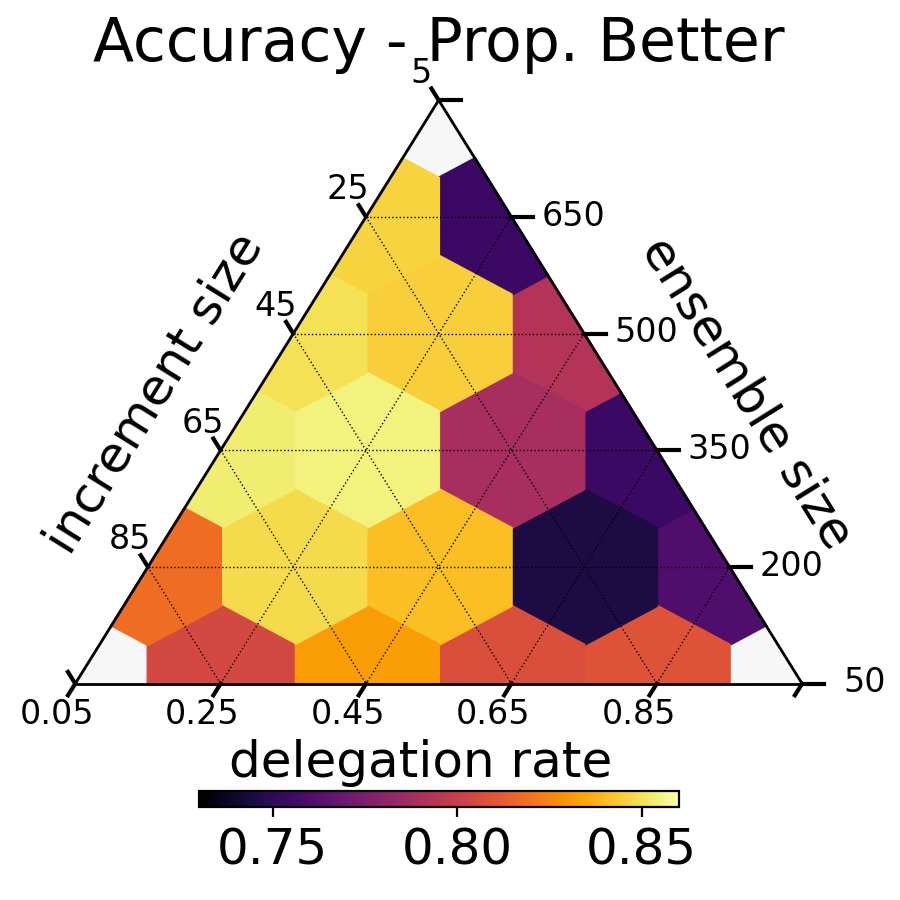}
     \end{subfigure}
     \hfill
     \begin{subfigure}[b]{0.45\textwidth}
         \centering
         \includegraphics[width=0.75\textwidth]{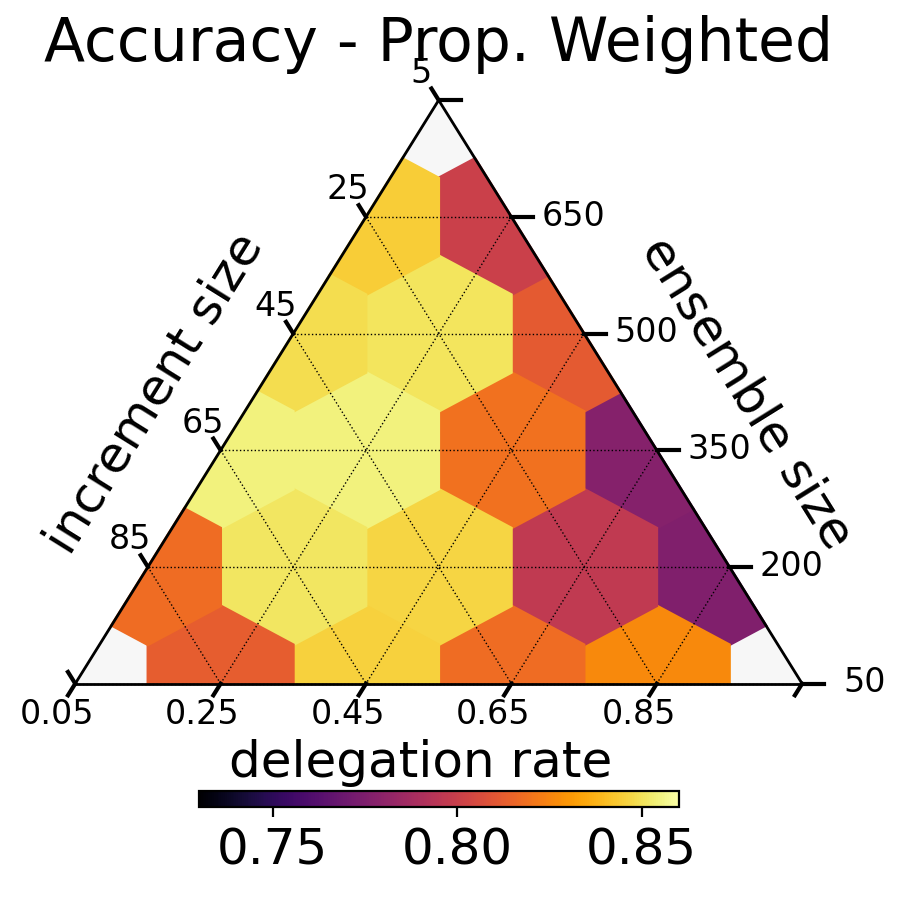}
     \end{subfigure}
     \caption{Test accuracy of fully trained ensembles as parameters varied. Results from ionosphere dataset.}
\end{figure}

\begin{figure}[ht!]
     \centering
     \begin{subfigure}[b]{0.49\textwidth}
         \centering
        \includegraphics[width=0.98\columnwidth, keepaspectratio]{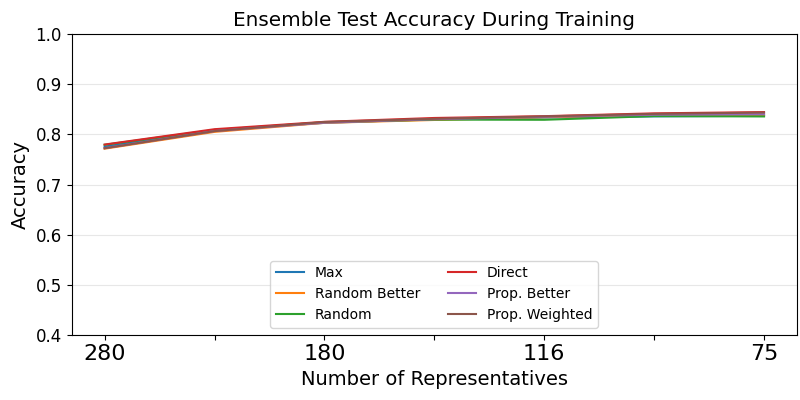}
     \end{subfigure}
     \hfill
     \begin{subfigure}[b]{0.49\textwidth}
        \includegraphics[width=0.98\columnwidth, keepaspectratio]{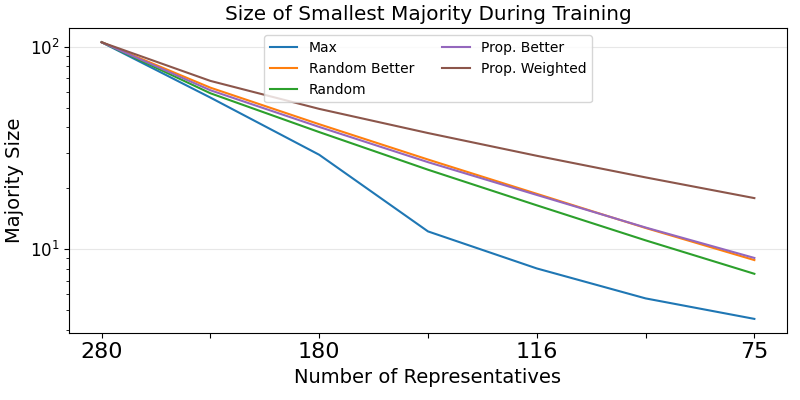}
     \end{subfigure}
     \caption{\textbf{(left)} Test accuracy during training on ionosphere dataset, averaged over 500 trials. \textbf{(right)} Minimum majority size during training on the ionosphere dataset.}
\end{figure}


\newpage

\subsection{kr-vs-kp}

\begin{figure}[ht!]
     \centering
     \begin{subfigure}[b]{0.45\textwidth}
         \centering
         \includegraphics[width=0.75\textwidth]{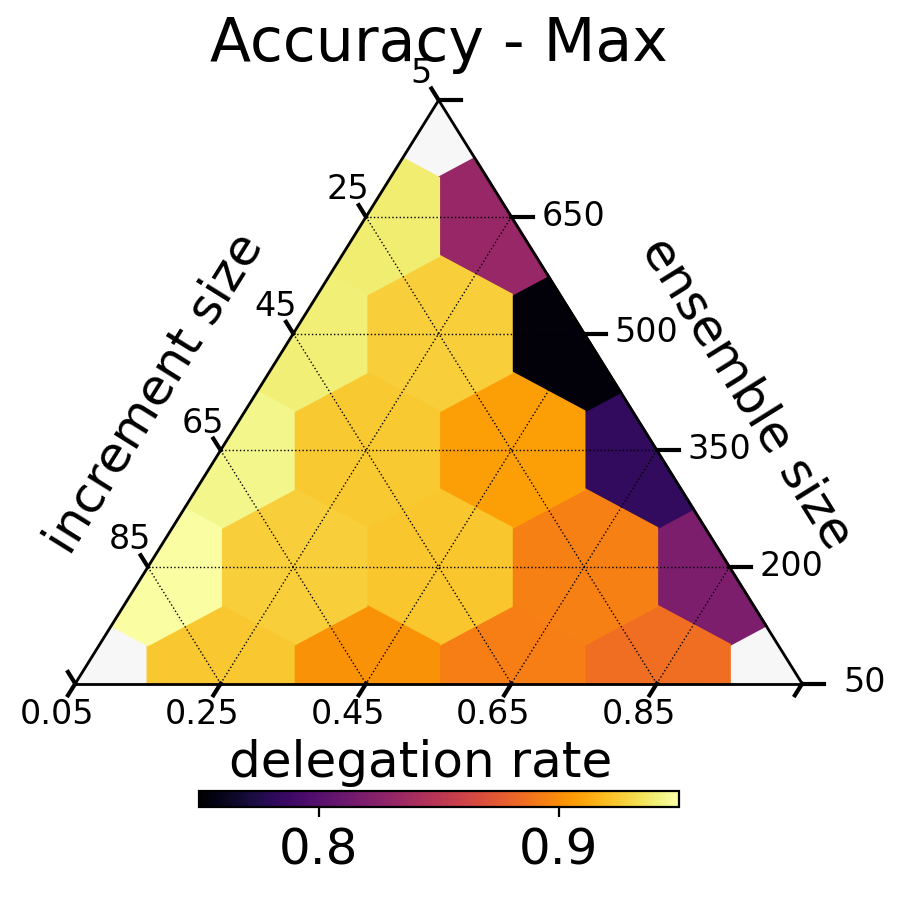}
     \end{subfigure}
     \hfill
     \begin{subfigure}[b]{0.45\textwidth}
         \centering
         \includegraphics[width=0.75\textwidth]{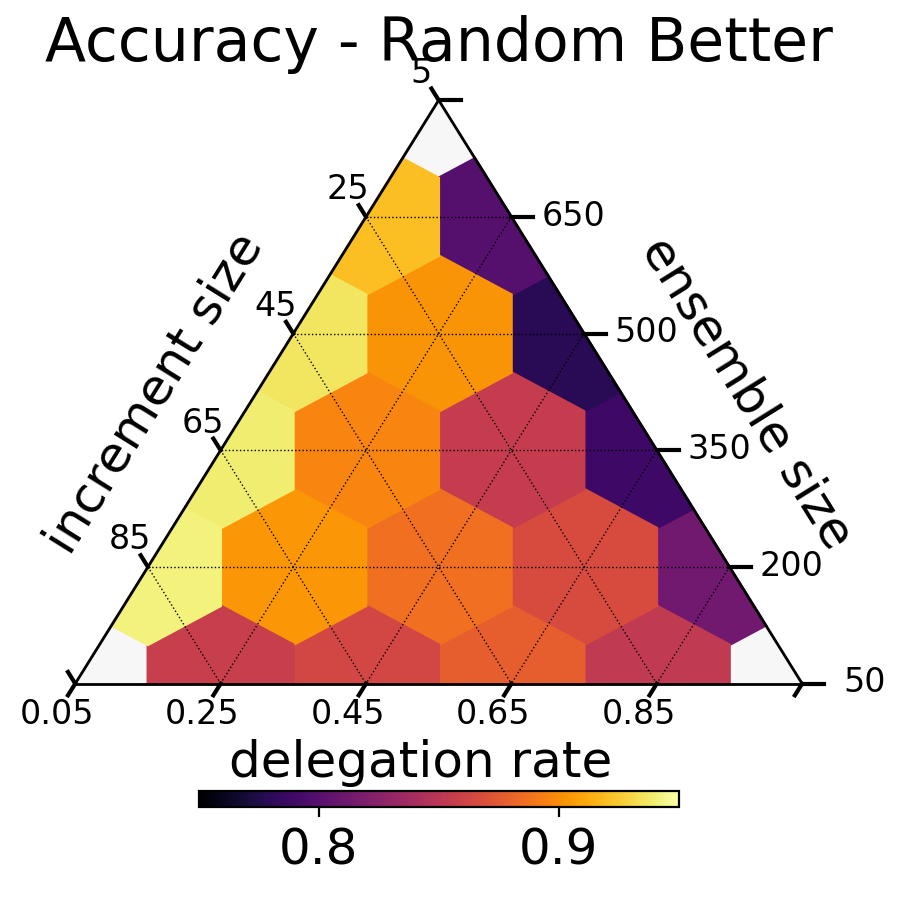}
     \end{subfigure}
     \\ 
     \begin{subfigure}[b]{0.45\textwidth}
         \centering
         \includegraphics[width=0.75\textwidth]{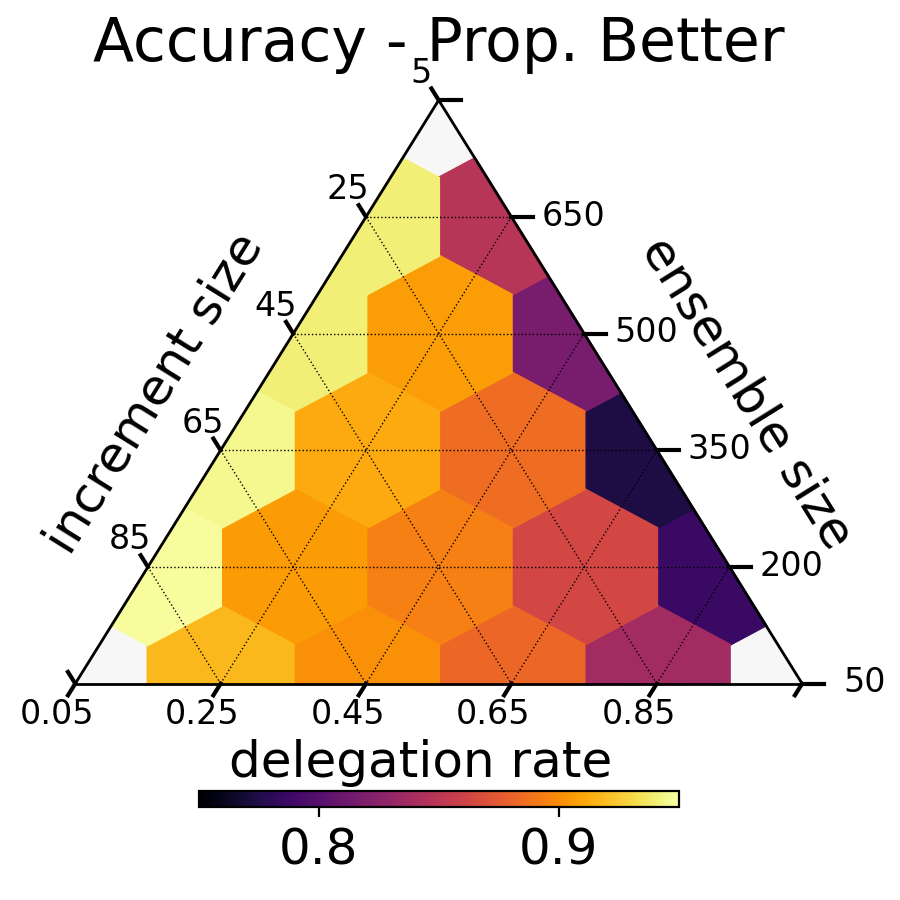}
     \end{subfigure}
     \hfill
     \begin{subfigure}[b]{0.45\textwidth}
         \centering
         \includegraphics[width=0.75\textwidth]{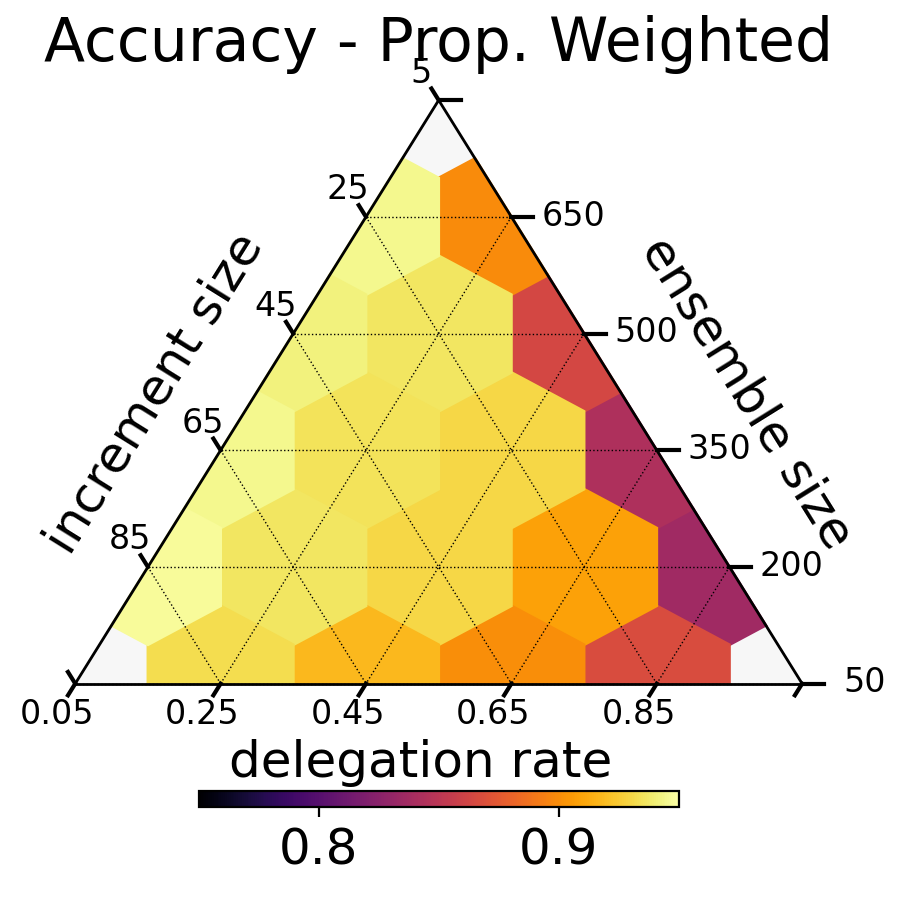}
     \end{subfigure}
     \caption{Test accuracy of fully trained ensembles as parameters varied. Results from kr-vs-kp dataset.}
\end{figure}

\begin{figure}[ht!]
     \centering
     \begin{subfigure}[b]{0.49\textwidth}
         \centering
        \includegraphics[width=0.98\columnwidth, keepaspectratio]{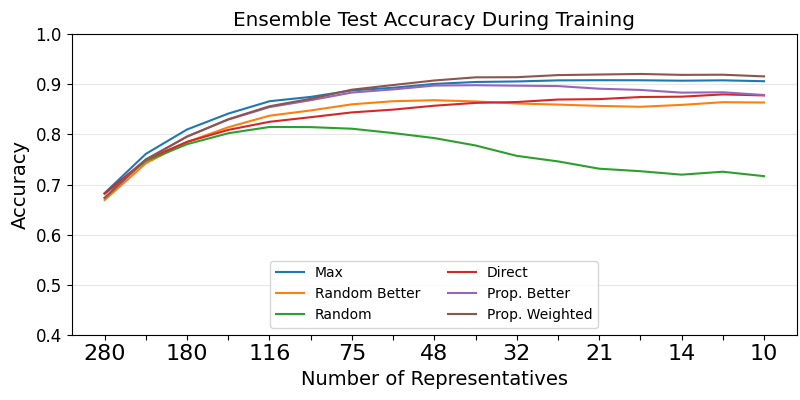}
     \end{subfigure}
     \hfill
     \begin{subfigure}[b]{0.49\textwidth}
         \centering
        \includegraphics[width=0.98\columnwidth, keepaspectratio]{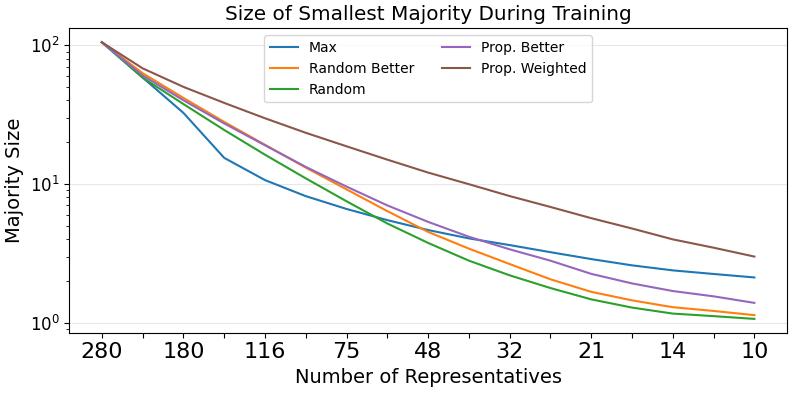}
     \end{subfigure}
     \caption{\textbf{(left)} Test accuracy during training on kr-vs-kp dataset, averaged over 500 trials. \textbf{(right)} Minimum majority size during training on the kr-vs-kp dataset.}
\end{figure}



\subsection{occupancy-detection}

\begin{figure}[ht!]
     \centering
     \begin{subfigure}[b]{0.45\textwidth}
         \centering
         \includegraphics[width=0.75\textwidth]{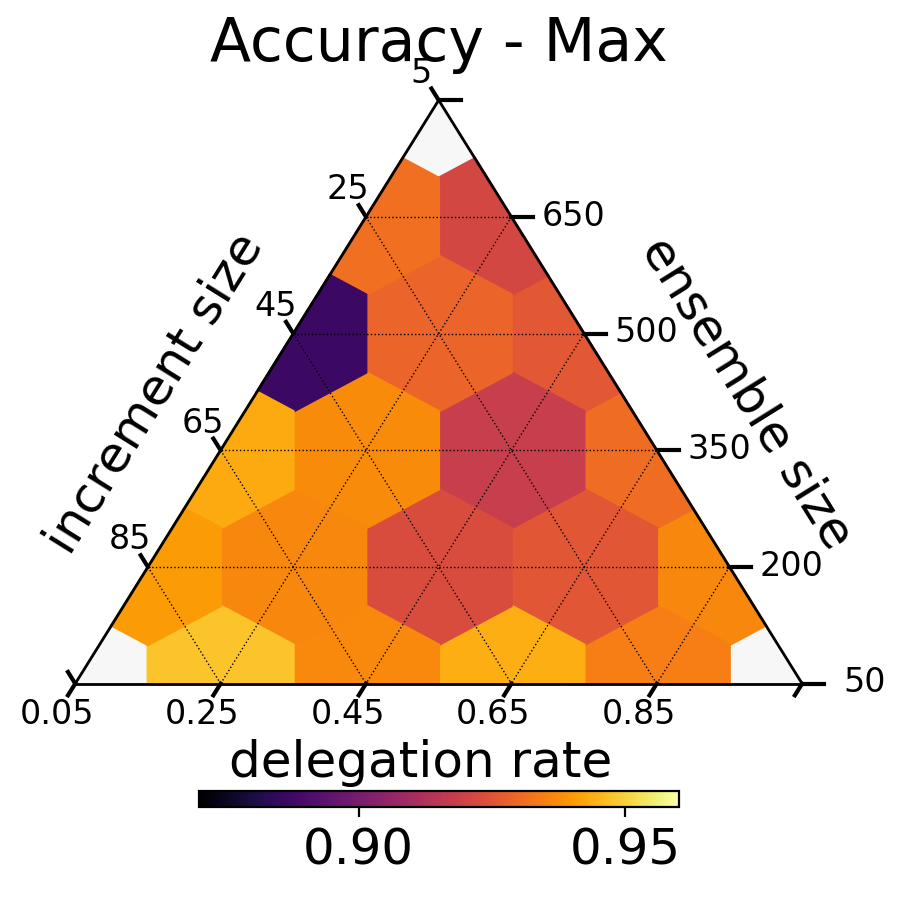}
     \end{subfigure}
     \hfill
     \begin{subfigure}[b]{0.45\textwidth}
         \centering
         \includegraphics[width=0.75\textwidth]{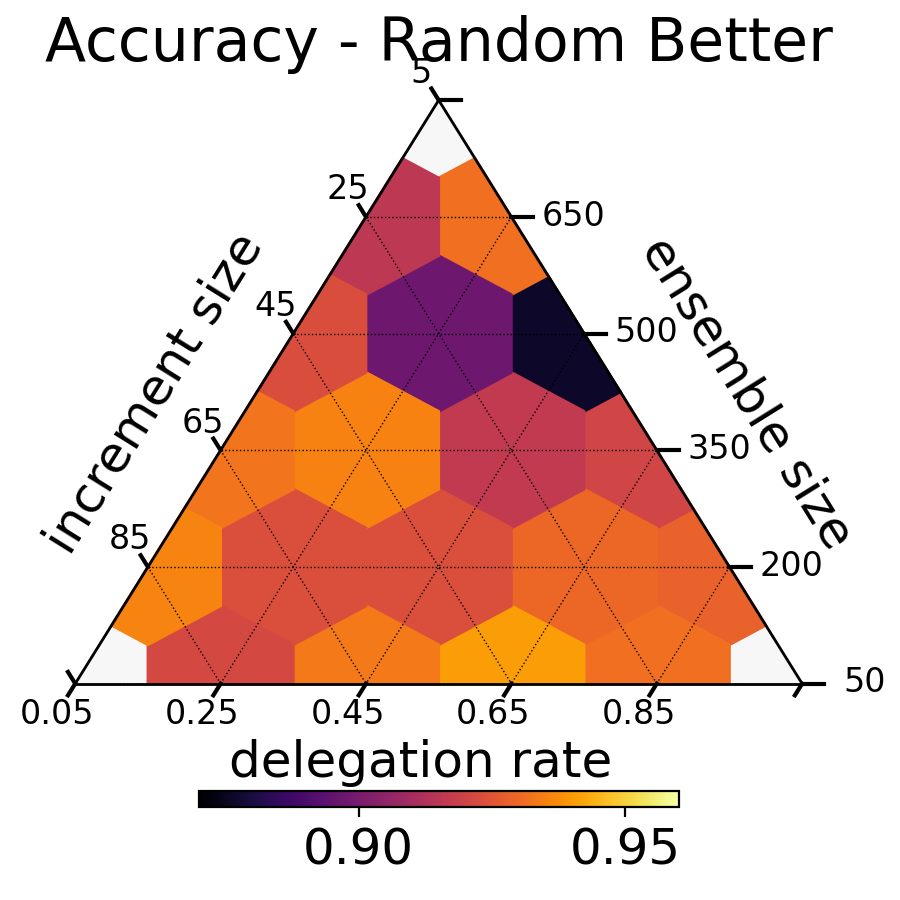}
     \end{subfigure}
     \\ 
     \begin{subfigure}[b]{0.45\textwidth}
         \centering
         \includegraphics[width=0.75\textwidth]{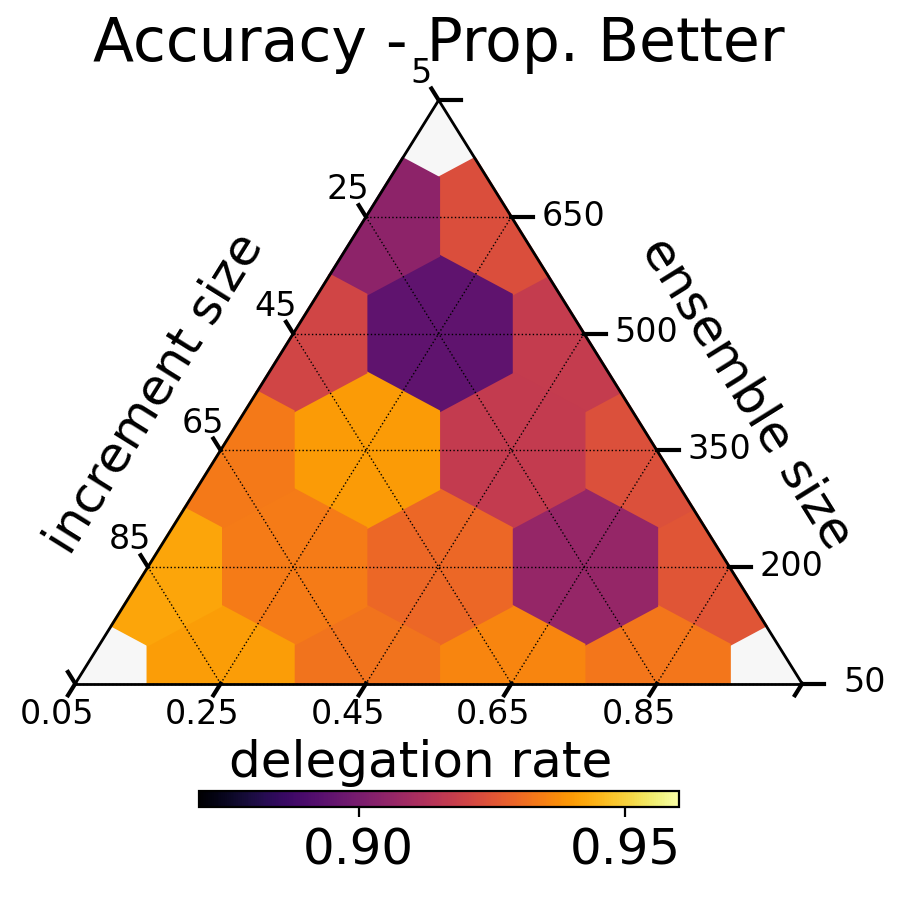}
     \end{subfigure}
     \hfill
     \begin{subfigure}[b]{0.45\textwidth}
         \centering
         \includegraphics[width=0.75\textwidth]{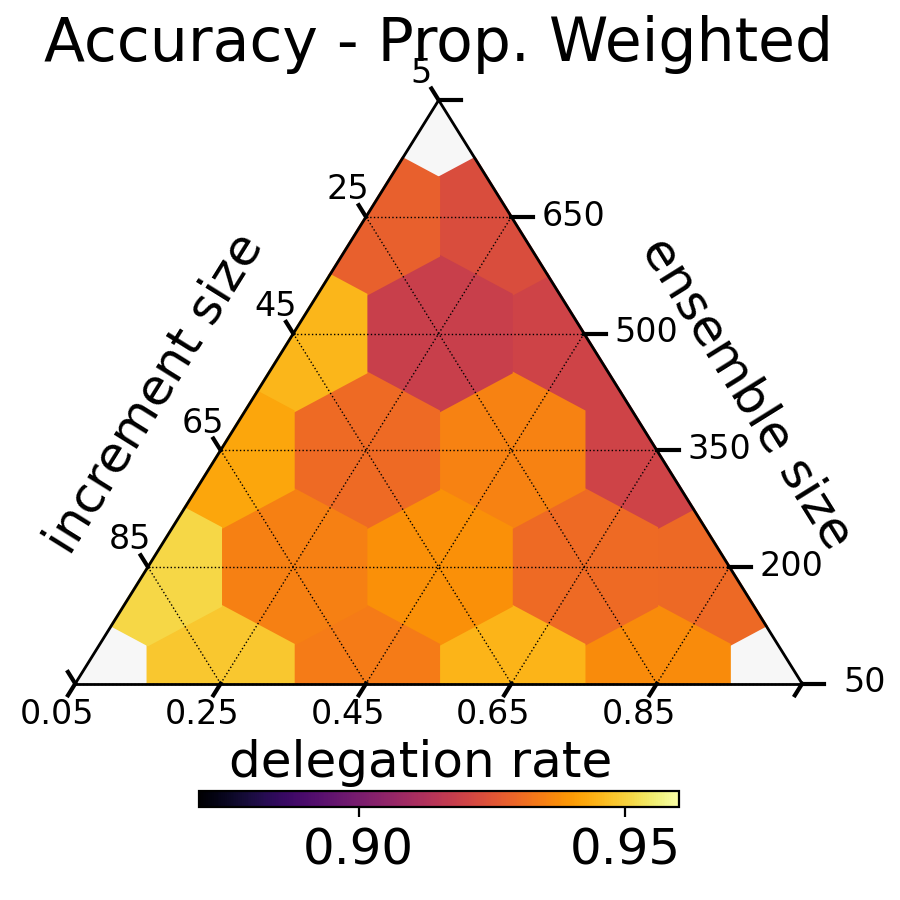}
     \end{subfigure}
     \caption{Test accuracy of fully trained ensembles as parameters varied. Results from occupancy-detection dataset.}
\end{figure}
\newpage

\subsection{online-shoppers}

\begin{figure}[ht!]
     \centering
     \begin{subfigure}[b]{0.45\textwidth}
         \centering
         \includegraphics[width=0.75\textwidth]{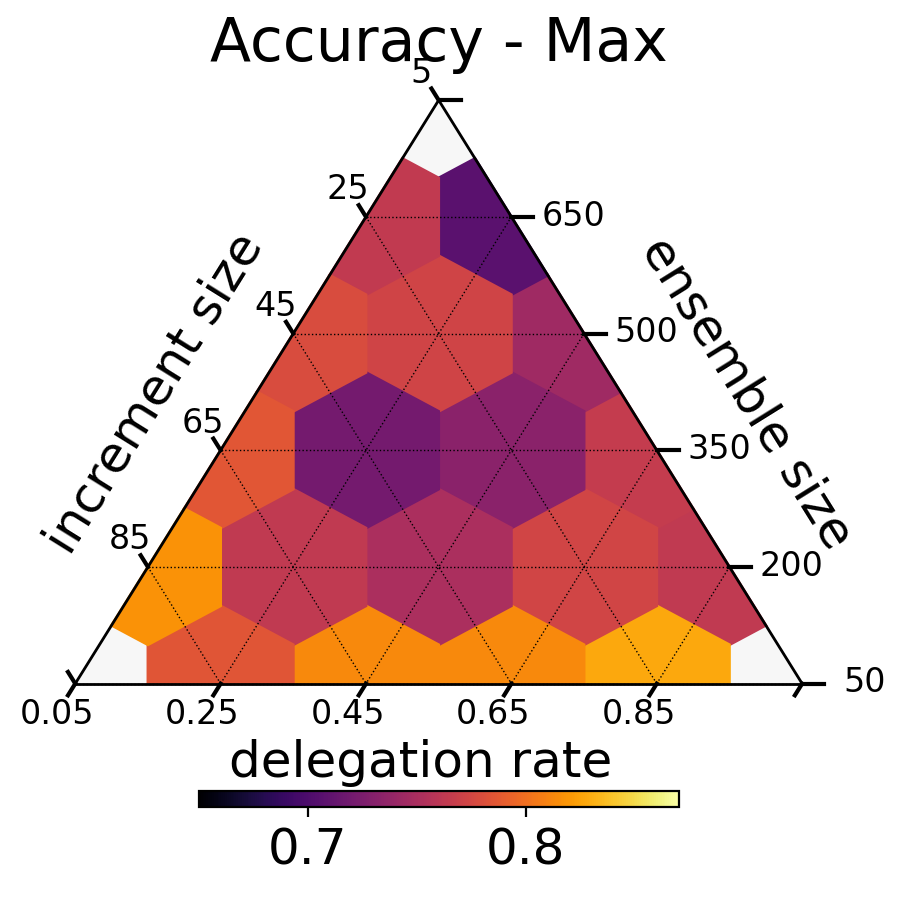}
     \end{subfigure}
     \hfill
     \begin{subfigure}[b]{0.45\textwidth}
         \centering
         \includegraphics[width=0.75\textwidth]{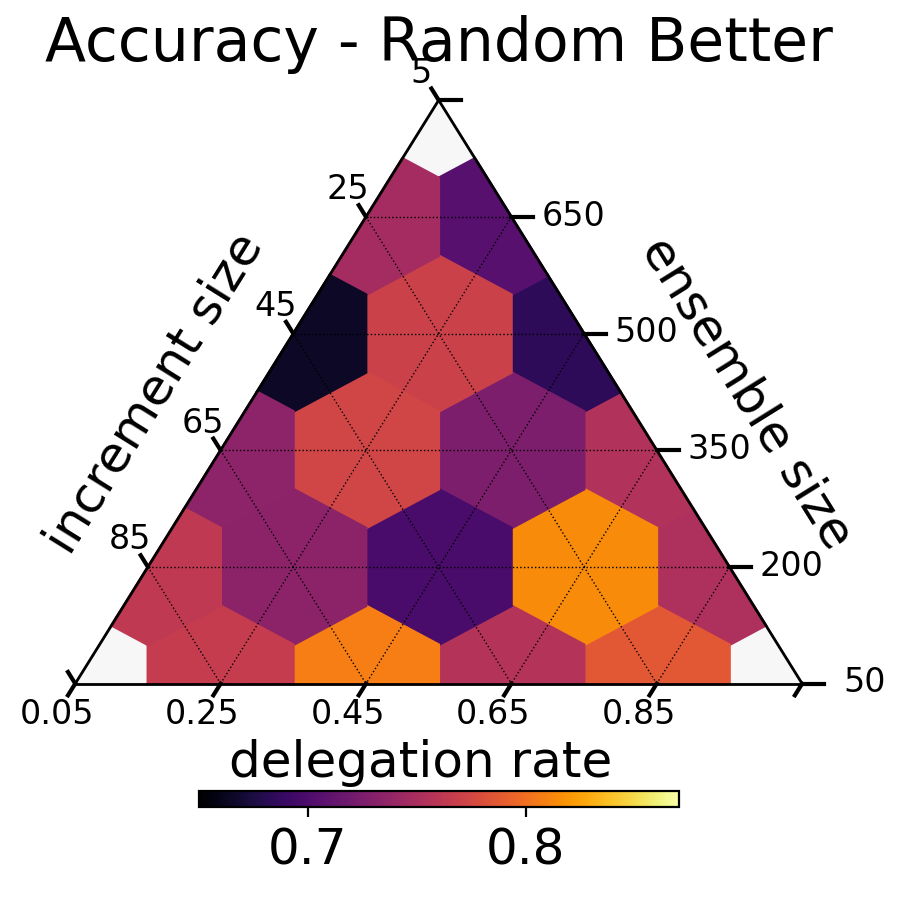}
     \end{subfigure}
     \\ 
     \begin{subfigure}[b]{0.45\textwidth}
         \centering
         \includegraphics[width=0.75\textwidth]{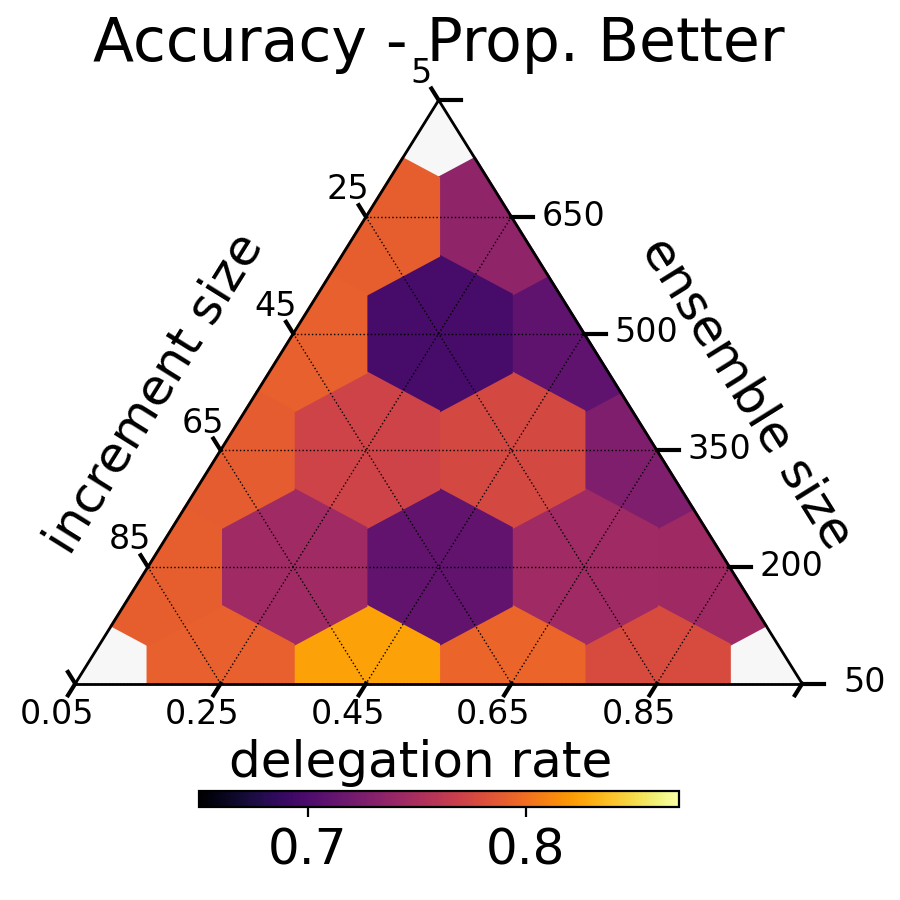}
     \end{subfigure}
     \hfill
     \begin{subfigure}[b]{0.45\textwidth}
         \centering
         \includegraphics[width=0.75\textwidth]{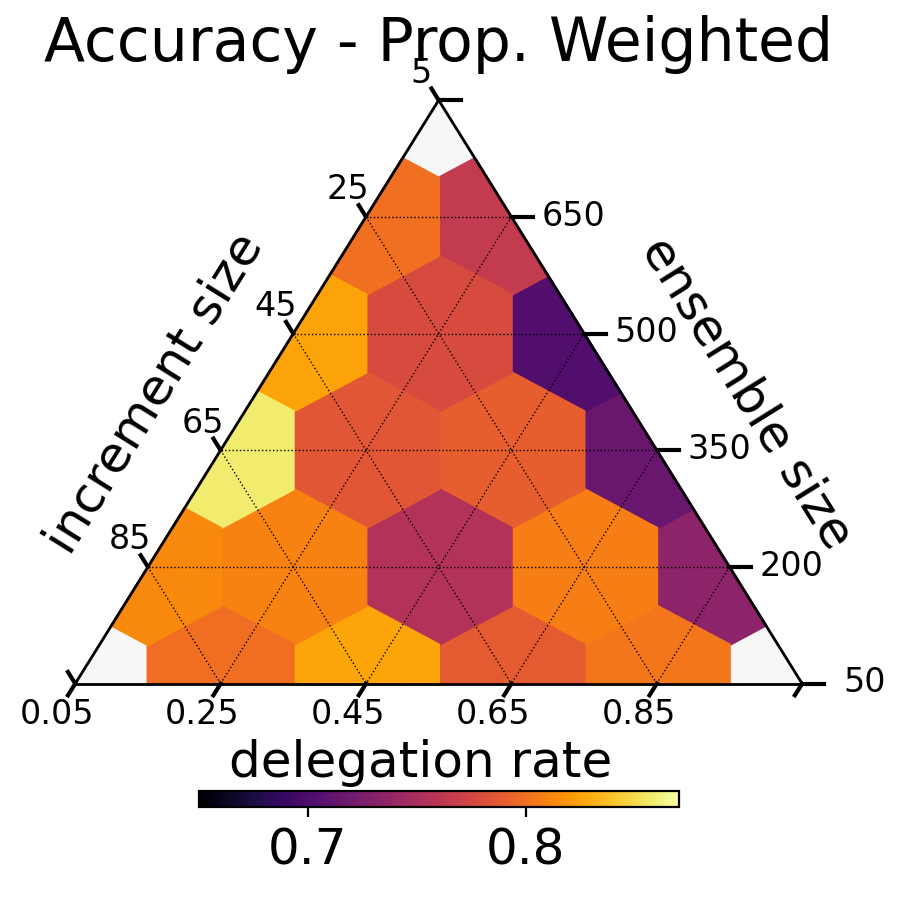}
     \end{subfigure}
     \caption{Test accuracy of fully trained ensembles as parameters varied. Results from online-shoppers dataset.}
\end{figure}

\begin{figure}[ht!]
     \centering
     \begin{subfigure}[b]{0.49\textwidth}
         \centering
        \includegraphics[width=0.98\columnwidth, keepaspectratio]{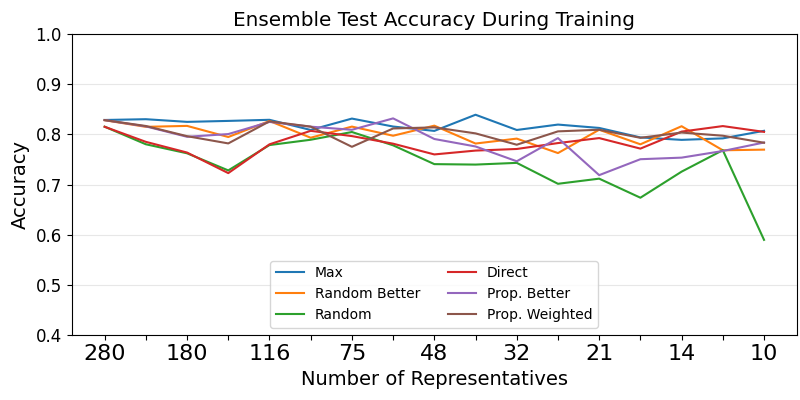}
     \end{subfigure}
     \hfill
     \begin{subfigure}[b]{0.49\textwidth}
         \centering
        \includegraphics[width=0.98\columnwidth, keepaspectratio]{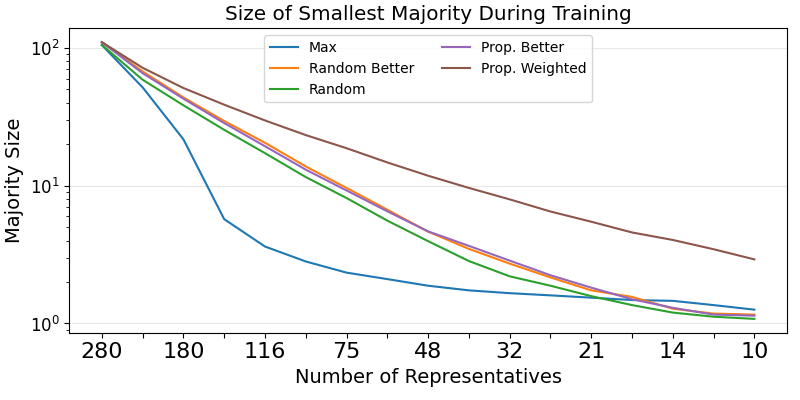}
     \end{subfigure}
     \caption{\textbf{(left)} Test accuracy during training on online-shoppers dataset, averaged over 500 trials. \textbf{(right)} Minimum majority size during training on the online-shoppers dataset.}
\end{figure}


\newpage

\subsection{spambase}

\begin{figure}[ht!]
     \centering
     \begin{subfigure}[b]{0.45\textwidth}
         \centering
         \includegraphics[width=0.75\textwidth]{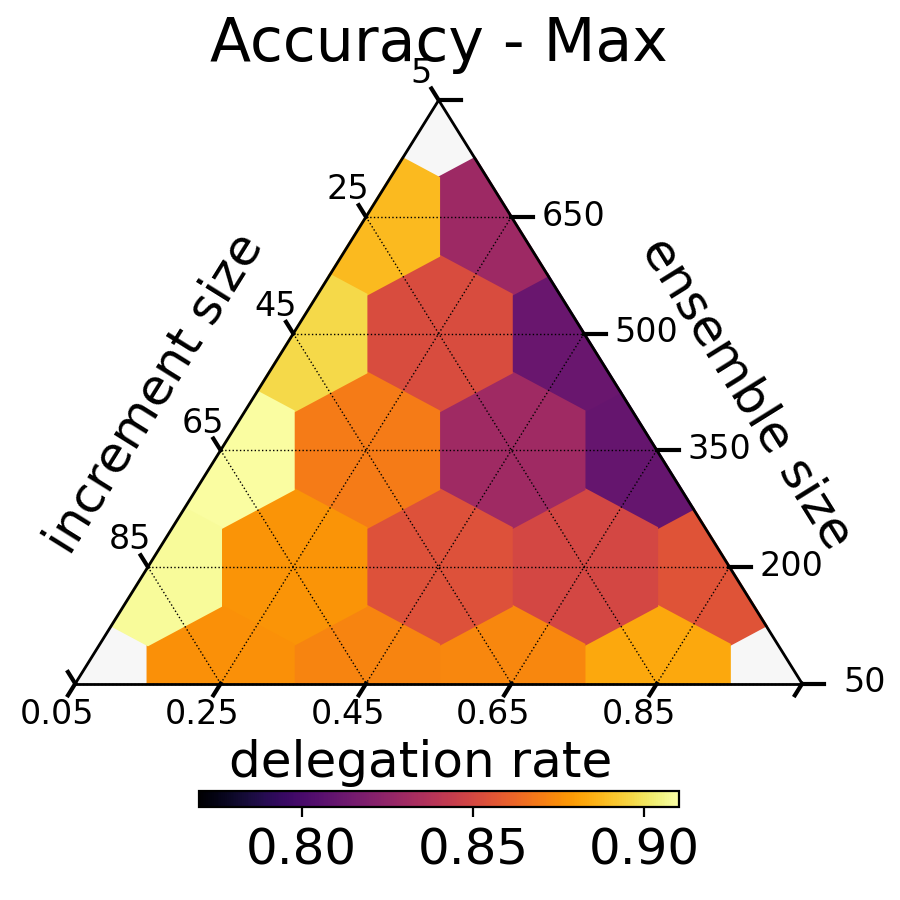}
     \end{subfigure}
     \hfill
     \begin{subfigure}[b]{0.45\textwidth}
         \centering
         \includegraphics[width=0.75\textwidth]{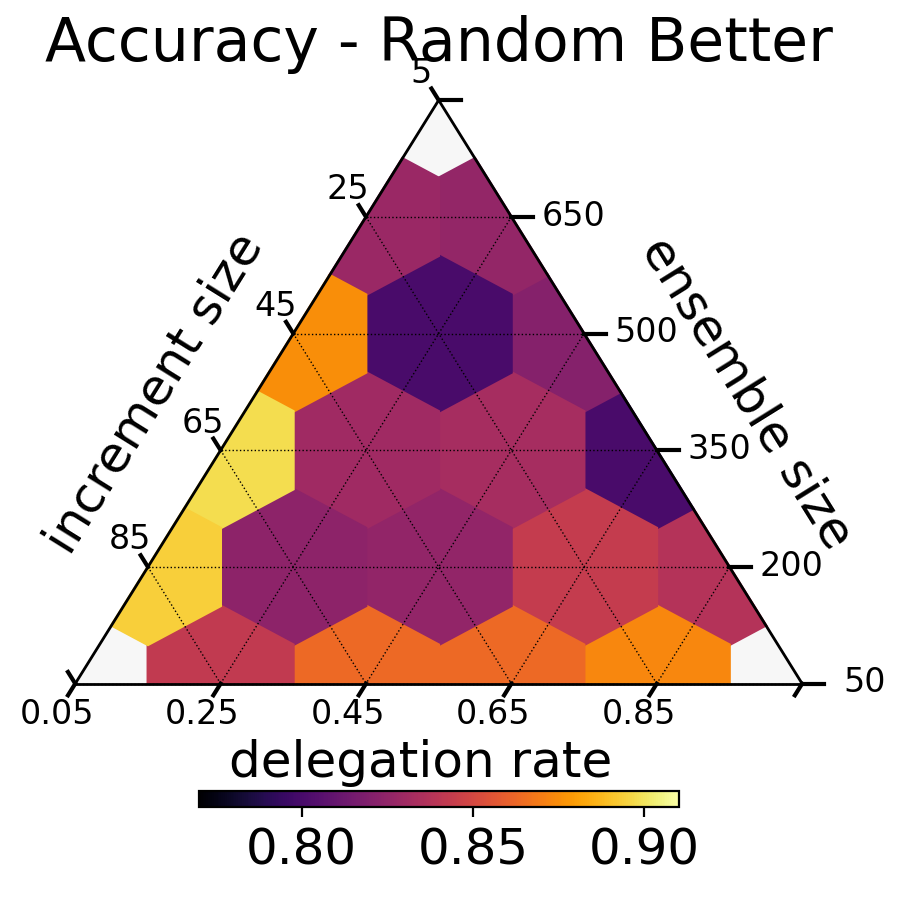}
     \end{subfigure}
     \\ 
     \begin{subfigure}[b]{0.45\textwidth}
         \centering
         \includegraphics[width=0.75\textwidth]{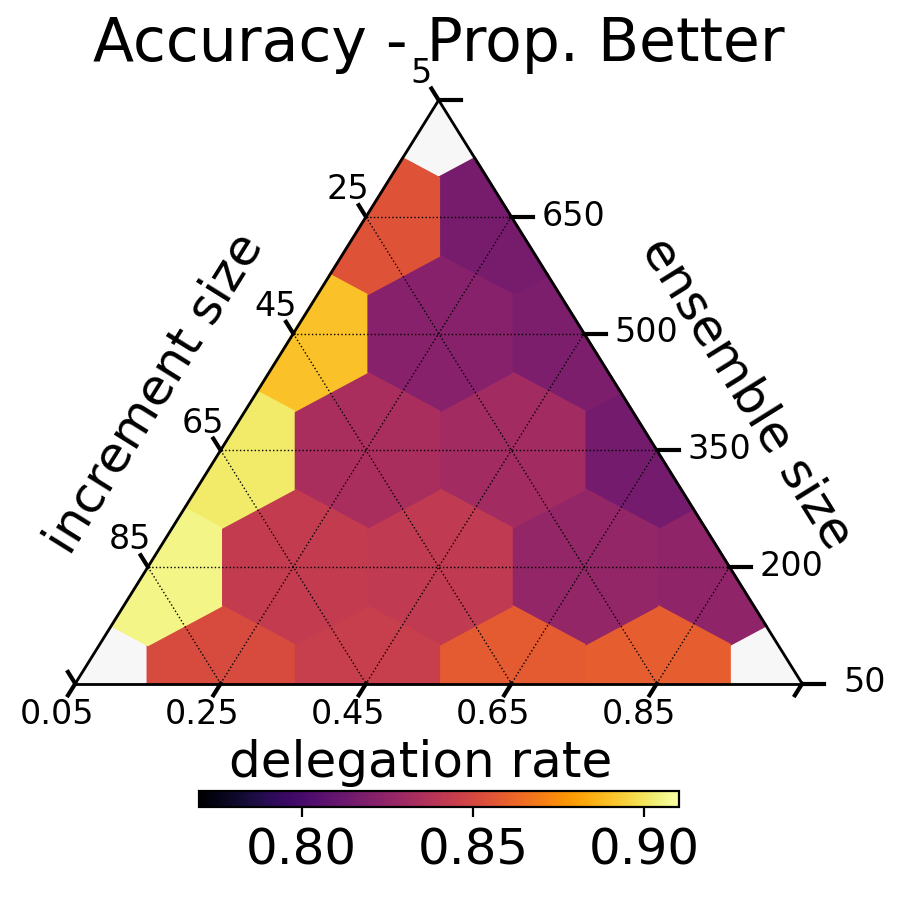}
     \end{subfigure}
     \hfill
     \begin{subfigure}[b]{0.45\textwidth}
         \centering
         \includegraphics[width=0.75\textwidth]{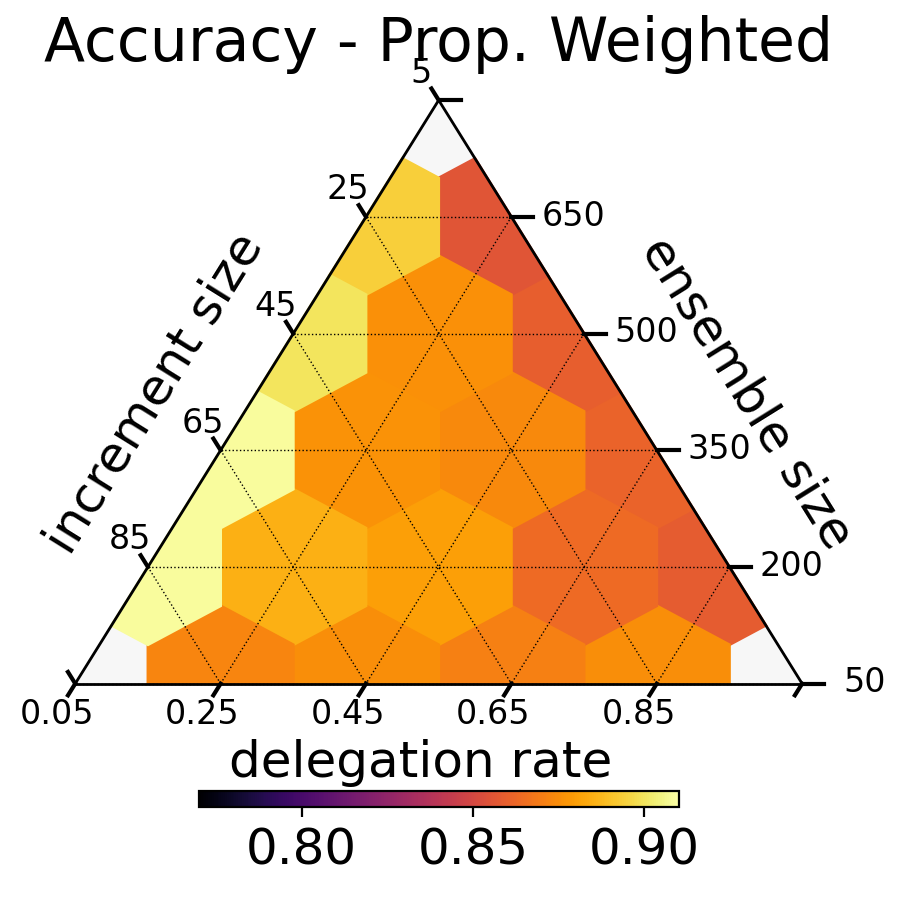}
     \end{subfigure}
     \caption{Test accuracy of fully trained ensembles as parameters varied. Results from spambase dataset.}
\end{figure}

\begin{figure}[ht!]
     \centering
     \begin{subfigure}[b]{0.49\textwidth}
         \centering
        \includegraphics[width=0.98\columnwidth, keepaspectratio]{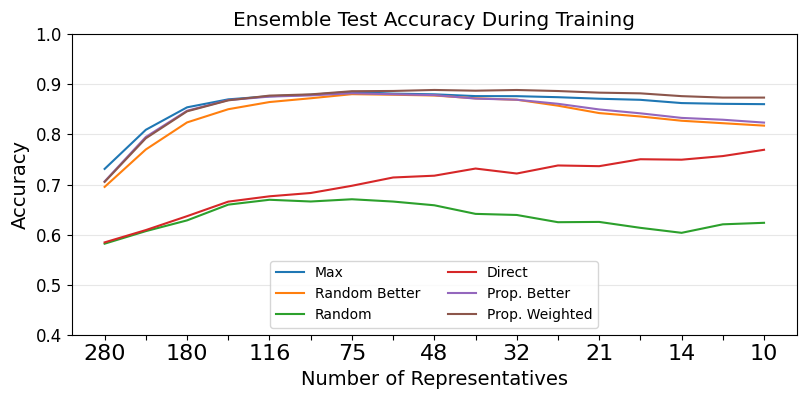}
     \end{subfigure}
     \hfill
     \begin{subfigure}[b]{0.49\textwidth}
         \centering
        \includegraphics[width=0.98\columnwidth, keepaspectratio]{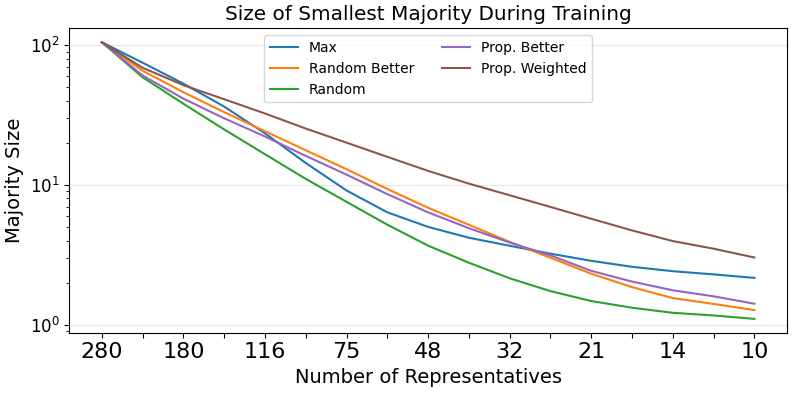}
     \end{subfigure}
     \caption{\textbf{(left)} Test accuracy during training on spambase dataset, averaged over 500 trials. \textbf{(right)} Minimum majority size during training on the spambase dataset.}
\end{figure}


\newpage

\section{Complete Accuracy and Cost Results}
\label{appendix:complete_acc_and_cost_results}

\autoref{tab:appendix-acc_maximizing} and \autoref{tab:appendix-cost_minimizing} show results for all delegations using, respectively, the accuracy maximizing and cost minimizing parameters described in \autoref{sec:comparison_with_adaboost}.

\begin{table*}[t]
\centering
\scriptsize
\begin{tabular}{llccccccccccc}
\hline
Ensemble      & \multicolumn{3}{c}{breast-cancer-w}                          & \multicolumn{3}{c}{credit-approval}        & \multicolumn{3}{c}{heart}                           & \multicolumn{3}{c}{ionosphere}          \\ \hline
              & Acc            & F1             & \multicolumn{1}{c|}{Cost}  & Acc   & F1    & \multicolumn{1}{c|}{Cost}  & Acc   & F1             & \multicolumn{1}{c|}{Cost}  & Acc            & F1             & Cost  \\
Direct Acc    & 0.907          & 0.926          & \multicolumn{1}{c|}{1}     & 0.628 & 0.585 & \multicolumn{1}{c|}{1}     & 0.584 & 0.59           & \multicolumn{1}{c|}{1}     & \textbf{0.854} & 0.765          & 1     \\
Max Acc       & 0.908          & 0.928          & \multicolumn{1}{c|}{0.782} & 0.632 & 0.588 & \multicolumn{1}{c|}{0.818} & 0.578 & 0.574          & \multicolumn{1}{c|}{0.901} & \textbf{0.852} & 0.761          & 0.895 \\
Rand B Acc    & 0.907          & 0.927          & \multicolumn{1}{c|}{0.782} & 0.628 & 0.584 & \multicolumn{1}{c|}{0.83}  & 0.569 & 0.571          & \multicolumn{1}{c|}{0.904} & \textbf{0.852} & 0.76           & 0.895 \\
Prop B Acc    & 0.908          & 0.928          & \multicolumn{1}{c|}{0.782} & 0.63  & 0.586 & \multicolumn{1}{c|}{0.817} & 0.573 & 0.571          & \multicolumn{1}{c|}{0.9}   & \textbf{0.853} & 0.763          & 0.894 \\
Prop W Acc    & 0.907          & 0.927          & \multicolumn{1}{c|}{0.782} & 0.631 & 0.586 & \multicolumn{1}{c|}{0.817} & 0.573 & 0.574          & \multicolumn{1}{c|}{0.9}   & \textbf{0.853} & 0.762          & 0.894 \\
Random Acc    & 0.907          & 0.926          & \multicolumn{1}{c|}{0.783} & 0.627 & 0.584 & \multicolumn{1}{c|}{0.816} & 0.586 & 0.597          & \multicolumn{1}{c|}{0.9}   & \textbf{0.854} & 0.767          & 0.895 \\
Ada DT Full   & 0.953          & 0.932          & \multicolumn{1}{c|}{0.039} & 0.818 & 0.833 & \multicolumn{1}{c|}{0.07}  & 0.758 & 0.783          & \multicolumn{1}{c|}{0.027} & 0.916          & 0.937          & 0.031 \\
Ada DT Small  & 0.957          & 0.938          & \multicolumn{1}{c|}{0.001} & 0.852 & 0.862 & \multicolumn{1}{c|}{0.002} & 0.803 & 0.824          & \multicolumn{1}{c|}{0.001} & 0.896          & 0.921          & 0.001 \\
Ada SGD Full  & 0.965          & 0.95           & \multicolumn{1}{c|}{0.014} & 0.653 & 0.681 & \multicolumn{1}{c|}{0.007} & 0.683 & 0.718          & \multicolumn{1}{c|}{0.01}  & 0.861          & 0.898          & 0.025 \\
Ada SGD Small & 0.965          & 0.95           & \multicolumn{1}{c|}{0.013} & 0.65  & 0.674 & \multicolumn{1}{c|}{0.007} & 0.678 & 0.718          & \multicolumn{1}{c|}{0.009} & 0.861          & 0.898          & 0.015 \\ \hline
Ensemble      & \multicolumn{3}{c}{kr-vs-kp}                                 & \multicolumn{3}{c}{occupancy-det}          & \multicolumn{3}{c}{online-shoppers}                 & \multicolumn{3}{c}{spambase}            \\ \hline
              & Acc            & F1             & \multicolumn{1}{c|}{Cost}  & Acc   & F1    & \multicolumn{1}{c|}{Cost}  & Acc   & F1             & \multicolumn{1}{c|}{Cost}  & Acc            & F1             & Cost  \\
Direct Acc    & \textbf{0.91}  & 0.903          & \multicolumn{1}{c|}{1}     & 0.946 & 0.964 & \multicolumn{1}{c|}{1}     & 0.869 & \textbf{0.927} & \multicolumn{1}{c|}{1}     & \textbf{0.86}  & \textbf{0.88}  & 1     \\
Max Acc       & \textbf{0.946} & \textbf{0.942} & \multicolumn{1}{c|}{0.272} & 0.925 & 0.95  & \multicolumn{1}{c|}{0.056} & 0.78  & \textbf{0.828} & \multicolumn{1}{c|}{0.059} & \textbf{0.909} & \textbf{0.927} & 0.197 \\
Rand B Acc    & \textbf{0.943} & \textbf{0.94}  & \multicolumn{1}{c|}{0.272} & 0.929 & 0.954 & \multicolumn{1}{c|}{0.056} & 0.719 & \textbf{0.764} & \multicolumn{1}{c|}{0.059} & \textbf{0.897} & \textbf{0.918} & 0.197 \\
Prop B Acc    & \textbf{0.946} & \textbf{0.942} & \multicolumn{1}{c|}{0.269} & 0.914 & 0.94  & \multicolumn{1}{c|}{0.055} & 0.784 & \textbf{0.84}  & \multicolumn{1}{c|}{0.058} & \textbf{0.905} & \textbf{0.924} & 0.198 \\
Prop W Acc    & \textbf{0.947} & \textbf{0.943} & \multicolumn{1}{c|}{0.269} & 0.94  & 0.96  & \multicolumn{1}{c|}{0.055} & 0.843 & \textbf{0.906} & \multicolumn{1}{c|}{0.058} & \textbf{0.909} & \textbf{0.927} & 0.198 \\
Random Acc    & 0.897          & 0.89           & \multicolumn{1}{c|}{0.309} & 0.906 & 0.923 & \multicolumn{1}{c|}{0.056} & 0.737 & \textbf{0.777} & \multicolumn{1}{c|}{0.058} & \textbf{0.795} & \textbf{0.807} & 0.198 \\
Ada DT Full   & 0.966          & 0.968          & \multicolumn{1}{c|}{0.02}  & 0.99  & 0.978 & \multicolumn{1}{c|}{0.009} & 0.888 & 0.605          & \multicolumn{1}{c|}{0.019} & 0.934          & 0.916          & 0.026 \\
Ada DT Small  & 0.946          & 0.948          & \multicolumn{1}{c|}{0.001} & 0.989 & 0.977 & \multicolumn{1}{c|}{0}     & 0.89  & 0.62           & \multicolumn{1}{c|}{0.001} & 0.916          & 0.891          & 0.001 \\
Ada SGD Full  & 0.941          & 0.944          & \multicolumn{1}{c|}{0.06}  & 0.984 & 0.966 & \multicolumn{1}{c|}{0.005} & 0.878 & 0.447          & \multicolumn{1}{c|}{0.012} & 0.786          & 0.719          & 0.013 \\
Ada SGD Small & 0.91           & 0.915          & \multicolumn{1}{c|}{0.01}  & 0.984 & 0.966 & \multicolumn{1}{c|}{0.005} & 0.879 & 0.444          & \multicolumn{1}{c|}{0.011} & 0.791          & 0.742          & 0.012 \\ \hline
\end{tabular}
\caption{Accuracy, F1 Score and Training Cost for all delegation mechanisms using accuracy maximizing parameters compared with each variety of Adaboost used. Bold values indicate when a delegating ensemble outperforms \textit{at least one} Adaboost method.}
\label{tab:appendix-acc_maximizing}
\end{table*}

\begin{table*}[t]
\centering
\scriptsize
\begin{tabular}{@{}lllllllllllll@{}}
\toprule
Ensemble      & \multicolumn{3}{c}{breast-cancer-w}                                                    & \multicolumn{3}{c}{credit-approval}                                                    & \multicolumn{3}{c}{heart}                                                     & \multicolumn{3}{c}{ionosphere}                                              \\ \midrule
              & \multicolumn{1}{c}{Acc} & \multicolumn{1}{c}{F1} & \multicolumn{1}{c|}{Cost}           & \multicolumn{1}{c}{Acc} & \multicolumn{1}{c}{F1} & \multicolumn{1}{c|}{Cost}           & \multicolumn{1}{c}{Acc} & \multicolumn{1}{c}{F1} & \multicolumn{1}{c|}{Cost}  & \multicolumn{1}{c}{Acc} & \multicolumn{1}{c}{F1} & \multicolumn{1}{c}{Cost} \\
Direct Cost   & 0.909                   & 0.927                  & \multicolumn{1}{l|}{1}              & 0.616                   & 0.587                  & \multicolumn{1}{l|}{1}              & 0.598                   & 0.625                  & \multicolumn{1}{l|}{1}     & 0.844                   & 0.758                  & 1                        \\
Max Cost      & 0.881                   & 0.897                  & \multicolumn{1}{l|}{\textbf{0.033}} & 0.612                   & 0.575                  & \multicolumn{1}{l|}{\textbf{0.036}} & 0.586                   & 0.62                   & \multicolumn{1}{l|}{0.035} & 0.799                   & 0.71                   & 0.034                    \\
Rand B Cost   & 0.87                    & 0.886                  & \multicolumn{1}{l|}{\textbf{0.033}} & 0.595                   & 0.563                  & \multicolumn{1}{l|}{\textbf{0.036}} & 0.539                   & 0.503                  & \multicolumn{1}{l|}{0.035} & 0.766                   & 0.68                   & 0.034                    \\
Prop B Cost   & 0.867                   & 0.886                  & \multicolumn{1}{l|}{\textbf{0.033}} & 0.593                   & 0.541                  & \multicolumn{1}{l|}{\textbf{0.036}} & 0.555                   & 0.517                  & \multicolumn{1}{l|}{0.035} & 0.754                   & 0.673                  & 0.034                    \\
Prop W Cost   & 0.9                     & 0.92                   & \multicolumn{1}{l|}{\textbf{0.033}} & 0.609                   & 0.584                  & \multicolumn{1}{l|}{\textbf{0.036}} & 0.565                   & 0.547                  & \multicolumn{1}{l|}{0.035} & 0.802                   & 0.716                  & 0.034                    \\
Random Cost   & 0.863                   & 0.881                  & \multicolumn{1}{l|}{\textbf{0.032}} & 0.595                   & 0.521                  & \multicolumn{1}{l|}{\textbf{0.037}} & 0.531                   & 0.484                  & \multicolumn{1}{l|}{0.035} & 0.723                   & 0.631                  & 0.033                    \\
Ada DT Full   & 0.956                   & 0.937                  & \multicolumn{1}{l|}{0.039}          & 0.825                   & 0.84                   & \multicolumn{1}{l|}{0.067}          & 0.751                   & 0.778                  & \multicolumn{1}{l|}{0.027} & 0.917                   & 0.937                  & 0.031                    \\
Ada DT Small  & 0.954                   & 0.934                  & \multicolumn{1}{l|}{0.001}          & 0.858                   & 0.868                  & \multicolumn{1}{l|}{0.002}          & 0.79                    & 0.814                  & \multicolumn{1}{l|}{0.001} & 0.896                   & 0.921                  & 0.001                    \\
Ada SGD Full  & 0.963                   & 0.947                  & \multicolumn{1}{l|}{0.014}          & 0.654                   & 0.681                  & \multicolumn{1}{l|}{0.008}          & 0.68                    & 0.709                  & \multicolumn{1}{l|}{0.009} & 0.861                   & 0.897                  & 0.027                    \\
Ada SGD Small & 0.963                   & 0.947                  & \multicolumn{1}{l|}{0.013}          & 0.657                   & 0.689                  & \multicolumn{1}{l|}{0.008}          & 0.679                   & 0.712                  & \multicolumn{1}{l|}{0.008} & 0.856                   & 0.894                  & 0.016                    \\ \midrule
Ensemble      & \multicolumn{3}{c}{kr-vs-kp}                                                           & \multicolumn{3}{c}{occupancy-det}                                                      & \multicolumn{3}{c}{online-shoppers}                                           & \multicolumn{3}{c}{spambase}                                                \\ \midrule
              & \multicolumn{1}{c}{Acc} & \multicolumn{1}{c}{F1} & \multicolumn{1}{c|}{Cost}           & \multicolumn{1}{c}{Acc} & \multicolumn{1}{c}{F1} & \multicolumn{1}{c|}{Cost}           & \multicolumn{1}{c}{Acc} & \multicolumn{1}{c}{F1} & \multicolumn{1}{c|}{Cost}  & \multicolumn{1}{c}{Acc} & \multicolumn{1}{c}{F1} & \multicolumn{1}{c}{Cost} \\
Direct Cost   & 0.9                     & 0.888                  & \multicolumn{1}{l|}{1}              & 0.938                   & 0.96                   & \multicolumn{1}{l|}{1}              & 0.82                    & \textbf{0.878}         & \multicolumn{1}{l|}{1}     & \textbf{0.854}          & \textbf{0.873}         & 1                        \\
Max Cost      & 0.885                   & 0.877                  & \multicolumn{1}{l|}{\textbf{0.026}} & 0.924                   & 0.951                  & \multicolumn{1}{l|}{0.029}          & 0.757                   & \textbf{0.81}          & \multicolumn{1}{l|}{0.029} & \textbf{0.859}          & \textbf{0.883}         & 0.03                     \\
Rand B Cost   & 0.845                   & 0.842                  & \multicolumn{1}{l|}{\textbf{0.026}} & 0.918                   & 0.945                  & \multicolumn{1}{l|}{0.029}          & 0.735                   & \textbf{0.781}         & \multicolumn{1}{l|}{0.029} & \textbf{0.843}          & \textbf{0.871}         & 0.03                     \\
Prop B Cost   & 0.856                   & 0.832                  & \multicolumn{1}{l|}{\textbf{0.026}} & 0.896                   & 0.913                  & \multicolumn{1}{l|}{0.029}          & 0.711                   & \textbf{0.751}         & \multicolumn{1}{l|}{0.029} & \textbf{0.835}          & \textbf{0.864}         & 0.029                    \\
Prop W Cost   & 0.908                   & 0.902                  & \multicolumn{1}{l|}{\textbf{0.026}} & 0.916                   & 0.936                  & \multicolumn{1}{l|}{0.029}          & 0.768                   & \textbf{0.817}         & \multicolumn{1}{l|}{0.029} & \textbf{0.869}          & \textbf{0.89}          & 0.029                    \\
Random Cost   & 0.721                   & 0.68                   & \multicolumn{1}{l|}{\textbf{0.03}}  & 0.894                   & 0.91                   & \multicolumn{1}{l|}{0.028}          & 0.78                    & \textbf{0.835}         & \multicolumn{1}{l|}{0.028} & 0.675                   & 0.641                  & 0.03                     \\
Ada DT Full   & 0.965                   & 0.967                  & \multicolumn{1}{l|}{0.02}           & 0.989                   & 0.977                  & \multicolumn{1}{l|}{0.009}          & 0.888                   & 0.605                  & \multicolumn{1}{l|}{0.019} & 0.934                   & 0.916                  & 0.026                    \\
Ada DT Small  & 0.945                   & 0.948                  & \multicolumn{1}{l|}{0.001}          & 0.989                   & 0.977                  & \multicolumn{1}{l|}{0}              & 0.89                    & 0.622                  & \multicolumn{1}{l|}{0.001} & 0.917                   & 0.893                  & 0.001                    \\
Ada SGD Full  & 0.94                    & 0.944                  & \multicolumn{1}{l|}{0.055}          & 0.984                   & 0.966                  & \multicolumn{1}{l|}{0.005}          & 0.879                   & 0.471                  & \multicolumn{1}{l|}{0.011} & 0.785                   & 0.734                  & 0.013                    \\
Ada SGD Small & 0.909                   & 0.916                  & \multicolumn{1}{l|}{0.01}           & 0.984                   & 0.966                  & \multicolumn{1}{l|}{0.005}          & 0.864                   & 0.432                  & \multicolumn{1}{l|}{0.011} & 0.784                   & 0.732                  & 0.012                    \\ \bottomrule
\end{tabular}
\caption{Accuracy, F1 Score and Training Cost for all delegation mechanisms using cost minimizing parameters compared with each variety of Adaboost used. Bold values indicate when a delegating ensemble outperforms \textit{at least one} Adaboost method.}
\label{tab:appendix-cost_minimizing}
\end{table*}

\end{document}